%% file: main.tex
\documentclass[11pt]{article}
\usepackage[T1]{fontenc}
\usepackage[square,sort,semicolon,numbers]{natbib}%[round]

\newcommand{\ucb}{\textsc{UCB}}
\newcommand{\UCB}{\textsc{UCB }}
\newcommand{\LFG}{\textsc{LFG }}
\newcommand{\lfg}{\textsc{LFG}}
\newcommand{\smab}{\textsc{MAB}}
\newcommand{\SMAB}{\textsc{MAB }}
\newcommand{\fsmab}{\textsc{Fair-MAB}}
\newcommand{\FSMAB}{\textsc{Fair-MAB }}

\newcommand{\tfucb}{\textsc{T-fair-ucb}}
\newcommand{\TFUCB}{\textsc{T-fair-ucb }}
\newcommand{\naive}{\textsc{Naive}}
\newcommand{\NAIVE}{\textsc{Naive }}
\newcommand{\fucb}{\textsc{Fair-ucb}}
\newcommand{\FUCB}{\textsc{Fair-ucb }}
\newcommand{\falg}{\textsc{Fair-Learn}}
\newcommand{\FALG}{\textsc{Fair-Learn }}

\newcommand{\rRegret}{\text{$r$-Regret }}
\newcommand{\rregret}{\text{$r$-Regret}}
\newcommand{\aFair}{$\alpha$-fair }

\newcommand{\learn}{{\fontfamily{cmss}\selectfont \small \textsc{Learn}}}

\newcommand{\TOL}{\emph{unfairness tolerance }}

\input{macro.tex}

\title{ Achieving Fairness in the Stochastic Multi-armed Bandit Problem }
\DeclareMathOperator*{\argmax}{argmax}

\begin{document}
\author{Vishakha Patil \thanks{Indian Institute of Science. \texttt{patilv@iisc.ac.in} } \quad Ganesh Ghalme\thanks{Indian Institute of Science.  \texttt{ganeshg@iisc.ac.in}} \quad Vineet Nair \thanks{Indian Institute of Science. \texttt{vineet@iisc.ac.in}}\\  \quad Y. Narahari\thanks{Indian Institute of Science. \texttt{narahari@iisc.ac.in}}}
\maketitle
\begin{abstract}
 We study an interesting variant of the stochastic multi-armed bandit problem, which we call the \FSMAB problem, where, in addition to the objective of maximizing the sum of expected rewards, the algorithm also needs to ensure that at any time, each arm is pulled  at least a pre-specified fraction of times. We investigate the interplay between \emph{learning} and  \emph{fairness} in terms of a pre-specified vector denoting the fractions of guaranteed pulls. We define a  \emph{fairness-aware regret}, which we call \rregret, that takes into account the above fairness constraints and extends  the conventional notion of regret in a natural way. Our primary contribution is to obtain a complete characterization of  a class of \FSMAB algorithms via two parameters: the unfairness tolerance and the learning algorithm used as a black-box. For this class of algorithms, we provide a fairness guarantee that holds uniformly over time, irrespective of the choice of the learning algorithm. Further, when the learning algorithm is \ucb1, we show that our algorithm achieves constant \rRegret for a large enough time horizon. Finally, we analyze the \emph{cost of fairness} in terms of the conventional notion of regret. We conclude by experimentally validating our theoretical results.
 
%  We study an interesting variant of the stochastic multi-armed bandit problem, called the FAIR-MAB problem, where each arm is required to be pulled for at least a given fraction of the total available rounds. We investigate the interplay between learning and fairness in terms of a pre-specified vector denoting the fractions of guaranteed pulls. We define a fairness-aware regret, called r-Regret, that takes into account the above fairness constraints and extends the conventional notion of regret in a natural way. Our primary contribution is to obtain a complete characterization of a class of FAIR-MAB algorithms via two parameters: the unfairness tolerance and the learning algorithm used as a black-box. For this class of algorithms, We provide a fairness guarantee that holds uniformly over time irrespective of the choice of a learning algorithm. Further, when the learning algorithm is UCB1, we show that our algorithm attains a constant r-Regret for sufficiently large time horizons. Finally, we analyze the cost of fairness in terms of the conventional notion of regret. We conclude by experimentally validating our theoretical results.
\end{abstract}
\input{introduction.tex}
\input{model.tex}
%
%\input{T_aware_Algo.tex}
\input{framework.tex}

\input{cost_of_fairness.tex}
\input{proofs.tex}
\input{simulation.tex}
\input{related_work.tex}
\input{discussion.tex}

\section*{Acknowledgement}
We thank Prof. Siddharth Barman for his useful insights and for helping us improve the presentation of the paper. We also thank Prof. Krishna Gummadi for pointing us to the Facebook lawsuit. Vishakha Patil is grateful for the travel support provided by Google via Google Student Travel Grant. Ganesh Ghalme gratefully acknowledges the travel support by Tata Trusts. Finally, the authors thank the anonymous reviewers for their helpful comments.
%\newpage 
\bibliographystyle{unsrtnat}
\bibliography{references.bib}
\input{appendix.tex}

\end{document}

%% file: macro.tex
\usepackage{amssymb,amsthm,amsmath,amssymb,wrapfig,dsfont,authblk}
\usepackage{thmtools,thm-restate}
\PassOptionsToPackage{hyphens}{url}
\usepackage{varioref}
\usepackage{verbatim}
\usepackage[usenames,svgnames,xcdraw,table]{xcolor}
\definecolor{DarkBlue}{rgb}{0.1,0.1,0.5}
\definecolor{DarkGreen}{rgb}{0.1,0.5,0.1}
\usepackage[backref=page]{hyperref}
\hypersetup{
     colorlinks   = true,
     linkcolor    = DarkBlue, % color of internal links
     urlcolor     = DarkBlue, % color of external links
	 citecolor    = DarkGreen % color of links to bibliography
}
\renewcommand*{\backref}[1]{}
\renewcommand*{\backrefalt}[4]{%
    \ifcase #1 (Not cited.)%
    \or        (Cited on page~#2)%
    \else      (Cited on pages~#2)%
    \fi}
\usepackage[margin=0.9in]{geometry}
\usepackage{enumerate}
\usepackage{graphicx}
\usepackage{subdepth}
\usepackage{enumitem}
\usepackage{tcolorbox}
\usepackage[linesnumbered, ruled]{algorithm2e}
\usepackage{mathtools}
\usepackage{caption}
\usepackage{bbm}
\usepackage{soul}
\usepackage{mathrsfs,amsfonts,dsfont,authblk}
\usepackage{thmtools,thm-restate}
\usepackage[capitalise,noabbrev]{cleveref}
\usepackage[font=footnotesize,labelsep=newline,justification=centering]{caption}
\crefname{property}{Property}{Properties}
\usepackage{nicefrac}
\usepackage{multicol}
\usepackage{enumerate}
\usepackage{libertine}
\usepackage{wrapfig}
\usepackage[backgroundcolor=blue!3!white]{todonotes}
\usepackage[labelfont={normalfont,bf},textfont=it]{caption}
\usepackage{tikz}
\usetikzlibrary{positioning}
\usetikzlibrary{calc}
\newtheorem{theorem}{Theorem}%[section]

\theoremstyle{definition}

\newtheorem{observation}{Observation}
\newtheorem{definition}{Definition}%[section]
\theoremstyle{remark}

\usepackage{mdframed}

\DeclarePairedDelimiter{\floor}{\lfloor}{\rfloor}

\usepackage{tikz}
\usetikzlibrary{shapes,arrows,decorations.markings}

%% file: introduction.tex
\section{Introduction}
\label{sec: introduction}
The multi-armed bandit (\smab) problem is a classic framework for sequential decision-making in uncertain environments. Starting with the seminal work of Robbins (1952), over the years,  a significant body of work has been developed to address both theoretical aspects and practical applications of this problem; see \cite{bubeck2012,lattimore2018bandit,slivkins2019introduction} for textbook expositions of the \SMAB problem. Indeed, the study of the MAB problem and its numerous variants continues to be a central pursuit in multiple fields such as online learning and reinforcement learning. In the \SMAB setup, at every round a decision maker (an online algorithm) is faced with $k$ choices, which correspond to unknown (to the algorithm) reward distributions. Each choice is referred to as an arm and when the decision maker pulls a specific arm she receives a reward drawn from the corresponding (a priori unknown) distribution\footnote{The arms which are not pulled do not give any reward.}. The goal of the decision maker is to maximize the cumulative reward in expectation accrued through a sequence of arm pulls, i.e. if the process repeats for $T$ rounds then in each round the decision maker selects an arm with the objective of maximizing the total expected reward\footnote{We study the standard setup in which T is not known upfront to the online algorithm.}. 

%The decision maker’s goal is to maximize the cumulative reward accrued through a sequence of arm pulls; specifically, suppose the process repeats for $T$ rounds, in each round the decision maker selects an arm with the objective of maximizing total, expected reward.\footnote{We study the standard setup in which T is not known upfront to the online algorithm.}  

%The multi-armed bandit problem is a classic framework sequential decision-making problem in uncertain environments, and was first described by  \cite{robbins1952}. In the classical stochastic multi-armed bandits (\smab) problem, a decision maker is faced with $k$ choices (henceforth referred to as \emph{arms}). At each time $t$, a decision maker decides which choice to select (referred to as pulling an arm). Once a decision maker pulls an arm,  she gets a random reward drawn from a fixed reward distribution unknown to her. The arms which are not pulled do not give any reward. The goal of a decision maker is to pull the arms, one in each round, so that the sum of the expected reward from $T$ pulls, at any time $T$, is maximized. %The challenge faced by the decision maker is famously known in literature as the  exploration vs. exploitation dilemma i.e. whether to  explore the arms to find the best arm in terms of expected rewards or to  pull an  arm  that has given the best average reward  so far. 

%The \SMAB problem is of great significance in the realm of online learning. 
Several variations of the \SMAB problem have been extensively studied in the literature. Various papers  study \SMAB problems with additional constraints which include bandits with knapsack constraints \cite{BAD13}, bandits with budget  constraints \cite{XIA15}, sleeping bandits \cite{KLE10,CHA17}, etc. In this paper we consider \fsmab, a variant of the \SMAB problem where, in addition to maximizing the cumulative expected reward, the algorithm also needs to ensure that uniformly (i.e., at the end of every round) each arm is pulled at least a pre-specified fraction of times. This imposes an additional constraint on the algorithm which we refer to as a \textit{fairness constraint}, specified in terms of a vector $r \in \mathbb{R}^k$. 
% (or equivalently minimizing the cumulative regret)

Formally, each component $r_i$ of the given vector $r$ specifies a \emph{fairness-quota} for arm $i$ and the online algorithm must ensure that for all time steps $t$ (i.e. uniformly), each arm $i$ is pulled at least $\lfloor r_i\cdot t \rfloor $ times in $t$ rounds. The goal of the online algorithm is to minimize expected regret while satisfying the fairness requirement of each arm. The expected regret in this setting, which we call \rregret, is computed with respect to the optimal \emph{fair} policy (see Definition \ref{def:rRegret}). We note that the difficulty of this problem is in satisfying these fairness constraints at the end of every round, which in particular ensures fairness even when the time horizon is unknown to the algorithm beforehand. 
%A variant of the \FSMAB problem was formulated in \cite{LLJ19} to address fairness in the context of online decision-making. Guaranteeing a minimum fraction of pulls to each arm uniformly over time ensures that no arm starves irrespective of its quality.

%We adapt the model studied in \cite{LLJ19} to the fundamental \SMAB setting and obtain stronger results by way of establishing non-asymptotic fairness guarantee and stronger regret bounds. 
It is relevant to note that the current work contributes to the long line of work in constrained variants of the \SMAB problem \cite{BAD13,KLE10,XIA15}.%We believe studying such constraints is fundamental to the line of work that considers the \SMAB problem with added constraints, and in this paper, we make an important contribution towards in the context of ensuring fairness.
% That is, uniformly for all time steps $t$, the fairness constraint entails that, each arm $i$ is pulled at least $r_it$ times in the first $t$ rounds.
The fairness constraints described above naturally capture many real-world settings wherein the arm pulls correspond to allocation of resources among agents with specified entitlements (quotas). The objective of ensuring a certain minimum allocation guarantee to each individual is, at times, at odds with the objective of maximizing efficiency, the classical goal of any learning algorithm. However, in many applications the allocation rules must consider such constraints in order to ensure fairness. The minimum entitlement over available resources secures the prerogative of individuals. For concreteness, we next present a motivating example. 
%are also natural in many real world resource allocation problems where the arms are heterogeneous individuals or agents with specified entitlement (quotas).
%In the context of the  \SMAB setting, fairness constraints ensure that no individual starves from the lack of opportunities irrespective of her quality. fundamental interests and rights
%We also point out here that these constraints capture the  \emph{veil of ignorance } doctrine of \cite{RAWLS71} wherein each individual has equal claim to the resource  in original position (refer \cite{FRE19,HEI18NIPS} for detailed discussion). 
%\subsection{Motivating Examples}

%\textbf{Personalized Ad Placement: }
The US Department of Housing and Urban Development recently sued Facebook for engaging in housing discrimination by targeting ads based on attributes such as gender, race, religion, etc. which are protected classes under the US law\footnote{\url{https://www.technologyreview.com/s/613274/facebook-algorithm-discriminates-ai-bias/}}. Facebook's algorithm that decides which ad should be shown to a particular user, inadvertently ends up discriminating because of the objective that it is trying to optimize. For example if the algorithm learns that it can generate more revenue by displaying an ad to more number of men as compared to women, then it would end up discriminating against women. The proposed \FSMAB model ensures that both men and women are shown the ad for at least a pre-specified fraction of the total number of ad displays, thereby preserving the fundamental right of equal access to opportunities.
In a way, the minimum fraction guarantee also  provides a  moral justification to the chosen allocation rule by evaluating it to be fair under the veil of ignorance \cite{RAWLS71} in which an allocation rule is considered as a hypothetical agreement among free and equal individuals unaware of the natural capabilities and  circumstantial advantages and biases they might have i.e. a socially agreed upon allocation in the original position (refer to \cite{FRE19,HEI18NIPS} for a detailed discussion).  
%The \FSMAB framework that we describe models this problem as follows. Each group in a protected class is represented by an arm in the \FSMAB setting. For example, gender being the protected class, consider a \FSMAB instance with two arms: One representing men, and the other representing women. With a classic \SMAB algorithm such as \ucb1, if the algorithm learns that it can generate more revenue by choosing the first arm (showing a particular ad to men), then the second arm will starve of  opportunities i.e. women will be shown this ad significantly few number of times as compared to men.  Additionally, the \emph{fairness} of the algorithm can be validated based on the actual outcome of such a policy.
%, for example the number of clicks on an ad, revenue generated by an ad, etc.,

The fairness model in this work naturally captures many resource allocation situations such as the sponsored ads on a search engine where each advertiser should be guaranteed a certain fraction of pulls in a bid to avoid monopolization of ad space; crowd-sourcing where each crowd-worker is guaranteed a fraction of tasks in order to induce participation; and a wireless communication setting where the receiver must ensure minimum quality of service to each sender. The work by \cite{LLJ19} contains a detailed discussion of these applications. We discuss other related works on fairness in Section \ref{sec:relWork}.  
%The reader is referred to the work by \cite{LLJ19} that considers the combinatorial sleeping \SMAB problem with similar constraints, for a detailed discussion of these applications. 

\textbf{Our contributions: } We first  define the \FSMAB problem in Section \ref{sec:model}. Any \FSMAB algorithm is evaluated based on two criteria: the fairness guarantee it  provides and its \rregret. The fairness notion that we consider requires that the fairness constraints be satisfied after each round, and the \rRegret notion is a natural extension of the conventional notion of regret which is defined with respect to an optimal policy that satisfies the fairness constraints. The uniform time fairness guarantee that we seek ensures fairness even in \emph{horizon-agnostic case}, that is when the time horizon $T$ is unknown to the algorithm. We remark that, even when the horizon $T$ is known, the intuitive approach of pulling each arm sufficiently many times to satisfy its fairness constraint does not guarantee fairness at the end of each round (see Appendix, Algorithm \ref{alg:tfucb}).%The reader is referred to the Appendix of the extended version\footnote{Extended version can be found at \url{https://arxiv.org/abs/1907.10516}}. 

%\textbf{Our contributions: } We first  define the \FSMAB problem in Section \ref{sec:model}. Any \FSMAB algorithm is evaluated based on two criteria: the fairness guarantee it  provides and its \rregret. In this work the we prove uniform time fairness guarantee, that is the fairness constraints are satisfied after every round whereas the \rRegret notion is a natural extension of the conventional notion of regret and is defined with respect to an optimal policy that satisfies the fairness constraints. The uniform (over time) fairness guarantee that we seek is significant in the horizon \emph{horizon-agnostic case}, that is when the time horizon $T$ is unknown to the algorithm. We remark that, even when the horizon $T$ is known, the intuitive approach of pulling each arm sufficiently many times to satisfy its fairness constraint, does not guarantee fairness at the end of each round (see Appendix, Algorithm \ref{alg:tfucb}).  

%Note that, with the knowledge of the time horizon $T$, it is easy to extend the existing algorithms to satisfy the fairness constraints at the end of $T$ rounds and at the same time achieve sub-linear \rregret. For this \emph{horizon aware} setting, we provide an algorithm in the Appendix (Algorithm \ref{alg:naive}).

%The challenging setting in our work is the \emph{horizon-agnostic case}, that is when the time horizon $T$ is unknown to the algorithm.
%, which makes the any-time fairness satisfaction problem significantly harder. 
As our primary contribution, in Section \ref{sec: proposed algorithms}, we define a class of \FSMAB algorithms, called \falg,
characterized by two parameters: the unfairness tolerance and the learning algorithm used as a black-box. We prove that any algorithm in \FALG satisfies the fairness constraints at any time $t$. Thus the fairness guarantee for \FALG holds uniformly over time, independently  of the choice of the learning algorithm. We note here that our meta-algorithm \falg, allows any \SMAB algorithm to be plugged-in as a black-box. This simple yet elegant framework can be implemented on top of any existing \SMAB algorithm to ensure fairness with quantifiable loss in terms of regret. The practical applicability of our algorithm is a notable feature of this work.  

When the learning algorithm is \ucb1, we prove a sub-logarithmic  \rRegret bound for the \FUCB algorithm. Additionally,  for sufficiently large $T$ we see that the  \FUCB incurs constant \rregret. We then evaluate the cost of fairness in \FSMAB with respect to the \emph{conventional} notion of regret in Section \ref{sec:cost_of_fair}. We conclude by providing detailed experimental results to validate our theoretical guarantees in Section \ref{sec:simulation}. In particular, we compare the performance of \FUCB with \LFG algorithm proposed in \cite{LLJ19}, which is the work closest to the current paper. We remark here that we obtain a much stronger fairness guarantee that holds at any time, unlike the asymptotic fairness guarantee of \LFG. We also prove a better regret bound with finer dependence on the problem instance parameters. Section \ref{sec:relWork} provides a detailed comparison.

%% file: model.tex
\section{The Model}\label{sec:model}
In this section we formally define the \FSMAB problem, the notion of fairness, and the concept of $r$-regret used in this work.
\subsection{The \FSMAB Problem}\label{sec:model problem}
An instance of the \FSMAB problem is a tuple $\langle T, [k], (\mu_i)_{i\in [k]}, (r_i)_{ i\in [k]}\rangle$, where $T$ is the time horizon, $[k] = \{1,2,\ldots,k\}$ is the set of arms, $\mu_i \in [0,1]$ represents the mean of the reward distribution $\mathcal{D}_i$ associated with arm $i$, and $(r_i)_{i \in [k]}$ represents the fairness constraint vector.
In the \FSMAB setting, the fairness constraints are exogenously specified to the algorithm in the form of a vector $r = (r_1,r_2,\ldots,r_k)$ where $r_i \in [0,1/k)$, for all $i \in [k]$, and consequently $\sum_{i \in [k]} r_i < 1$ and $r_i$ denotes the minimum fraction of times an arm $i \in [k]$ has to be pulled in $T$ rounds, for any $T$. We consider $r_i \in [0,1/k)$ to be consistent with the notion of \emph{proportionality} wherein, guaranteeing any arm a fraction greater than its proportional fraction, which is $1/k$, is \emph{unfair} in itself. 
%Note that our \rRegret guarantees hold for any $(r_1,r_2,\ldots,r_k)$ such that $r_i \in [0,1]^k$ where $\sum_{i\in [k]} r_i \leq 1$.

%Given a fairness constraint vector $r = (r_1,r_2,\ldots,r_k)$, $r_i$ is the fairness constraint for arm $i$ and denotes the minimum fraction of times arm $i$ needs to be pulled in $T$ rounds, for any $T$. Note that $r_i \in [0, 1]$ and $\sum_{i \in [k]} r_i \leq 1$.

In each round $t$, a \FSMAB algorithm pulls an arm $i_t \in [k]$ and collects the reward $X_{i_t} \sim \mathcal{D}_{i_t}$. We assume that the reward distributions are $Bernoulli({\mu}_i)$ for each arm $i \in [k]$. This assumption holds without loss of generality since one can reduce the \SMAB problem with general distributions supported on [0,1] to a \SMAB problem with Bernoulli rewards using the extension provided in \cite{AGR12}. Note that the true value of $\mu = ({\mu}_1, {\mu}_2,\ldots, {\mu}_k)$ is \emph{unknown} to the algorithm. Throughout this paper we assume without loss of generality that $\mu_1 > \mu_2 > \ldots > \mu_k$ and arm $1$ is called the \emph{optimal} arm.
 Next, we formalize the notions of fairness and regret used in the paper.
\subsection{Notion of Fairness}
\label{subsec:model fairness}
%\noindent\textbf{Notion of Fairness}\\
Let $N_{i,t}$ denote the number of times arm $i$ is pulled in $t$ rounds.  We first present the definition of fairness proposed by \cite{LLJ19} and then define the stronger notion of fairness considered in this paper.
\begin{definition}\cite{LLJ19}\label{def:asymptotic fairness}
A \FSMAB algorithm $\mathcal{A}$ is called (asymptotically) fair if $ \liminf_{t \rightarrow \infty } \mathbbm{E}_{\mathcal{A}}\big[ r_i - \frac{N_{i,t}}{t}\big] \leq 0$ for all $i \in [k]$. 
\end{definition}
We refer to the above notion of fairness as \emph{asymptotic fairness}. We now define a much stronger notion of fairness that holds  over all rounds and is parameterized by  the \TOL allowed in the system which is denoted by a constant $\alpha \geq 0$.
\begin{definition}\label{definition: fairness}
Given an \TOL $\alpha \geq 0$, a \FSMAB algorithm $\mathcal{A}$ is said to be \aFair if $\floor{r_it} - N_{i,t} \leq \alpha$ for all $t \leq T$ and  for all arms $i \in [k]$.
\end{definition}
In particular, if the above guarantee holds for $\alpha  = 0$, then we call the \FSMAB algorithm \emph{fair}. Note that our fairness guarantee holds uniformly over the time horizon and for any sequence of arm pulls $(i_t)_{t \leq T}$ by the algorithm. Hence it is much stronger than the guarantee in \cite{LLJ19} which only guarantees asymptotic fairness (Definition \ref{def:asymptotic fairness}). Notice that for any given constant $\alpha \geq 0$, $\alpha$-fairness (Definition \ref{definition: fairness}) implies asymptotic fairness.
\subsection{Notions of Regret}
\label{subsec:model regret}
%\noindent\textbf{Notions of Regret}\\
 In the \SMAB setting, the optimal policy is the one which pulls the optimal arm in every round. The regret of a \SMAB algorithm is defined as the difference between the cumulative reward of the optimal policy and that of the algorithm.
\begin{definition}
\label{def:smab regret}
The expected regret of a \SMAB algorithm $\mathcal{A}$ after $T$ rounds is defined as:  
\begin{equation}
    \mathcal{R}_{\mathcal{A}}(T) = \sum_{i\in [k]}{\Delta_i \cdot \mathbb{E}[N_{i,T}]}
    %\mathcal{R}_{\mathcal{A}}(T) = {\mu_1}.T - \mathbb{E}_{\mathcal{A}}\Big[\sum_{t \in [T]}X_{i_t}\Big]
\end{equation}
\end{definition}
where $\Delta_i = \mu_1 - \mu_i$ and $N_{i,T}$ denotes the number of pulls of an arm $i \in [k]$ by $\mathcal{A}$ in $T$ rounds.
%The expected regret of $\mathcal{A}$ can equivalently be written in terms of the expected number of pulls of the sub-optimal arms and the expected regret incurred due to playing the sub-optimal arms. In particular, if $\Delta_i = \mu_1 - \mu_i$ and $N_{i,T}$ denotes the number of pulls of an arm $i \in [k]$ by $\mathcal{A}$ in $T$ rounds, then the expected regret of $\mathcal{A}$ after $T$ rounds is defined as:
%\begin{equation}\label{eq:smab regret}
%    \mathcal{R}_{\mathcal{A}}(T) = \sum_{i\in [k]}{\Delta_i \cdot \mathbb{E}[N_{i,T}]}
%\end{equation}

We call an algorithm optimal if it attains zero regret. It is easy to see that the above notion of regret does not adequately quantify the performance of a \FSMAB algorithm as the optimal policy here does not account for the fairness constraints. Also, note that the conventional regret in the \FSMAB setting can be $O(T)$ (see Section \ref{sec:cost_of_fair} for further details). Hence, we first state the fairness-aware optimal policy that we consider as a baseline.
\begin{comment}
\begin{definition}\label{def:fairness aware optimal policy}
Given a fairness constraint vector $r = (r_1, r_2, \ldots, r_k)$ and the \TOL $\alpha \geq 0$, and knowing the value of $\mu = (\mu_1,\mu_2,\ldots,\mu_k)$, an optimal policy is the one that pulls each sub-optimal arm $i \neq 1$ exactly $\floor{r_iT} - \alpha$ number of times, and pulls the optimal arm $i = 1$ for the remaining $T - \sum_{i \neq 1} \big(\floor{r_iT} - \alpha \big)$ rounds.
\end{definition}
\end{comment}
\begin{observation}
\label{obs:observationOne}
A \FSMAB algorithm $\mathcal{A}$ is optimal  iff  $\mathcal{A}$ satisfies the following: if $\floor{r_i T} - \alpha >0$ then $N_{i,T} = \floor{r_i T} - \alpha $, else $N_{i,T} = 0$, for all $i \neq 1$.
\end{observation}

From Observation \ref{obs:observationOne} we have that an optimal \FSMAB algorithm that knows the value of $\mu$ must play sub-optimal arms exactly $\floor{ r_i\cdot T} -\alpha$ times in order to satisfy the fairness constraint and  play the optimal arm (arm 1) for the rest of the rounds i.e. for  $T - \sum_{i \neq 1} \floor{r_i\cdot T} + (k-1)\alpha$ rounds. The regret of an algorithm is compared with such an optimal policy that satisfies the fairness constraints in the \FSMAB setting.

\begin{definition}\label{def:rRegret}
Given a fairness constraint vector $r = (r_1, r_2, \ldots,r_k)$ and the \TOL $\alpha \geq 0$, the fairness-aware \rRegret of a \FSMAB algorithm $\mathcal{A}$ is defined as:
\begin{equation}
    \label{eq:rRegret}
    \mathcal{R}_{\mathcal{A}}^{r}(T) = \sum_{i\in [k]}{\Delta_i \cdot \Big(\mathbb{E}[N_{i,T}] - \textnormal{max}\big(0, \floor{r_i \cdot T} - \alpha \big) \Big)}
\end{equation}
\end{definition}
The $\textnormal{max}(0, \floor{r_i \cdot T} - \alpha )$ in the above definition accounts for the number of pulls of arm $i$ made by the optimal algorithm to satisfy its fairness constraint. Also the \rRegret of an algorithm that is not $\alpha$-fair could be negative but this is an infeasible solution. %Note that the  $\floor{r_iT} - \alpha$ pulls of any sub-optimal arm $i$ do not incur any \rRegret, as the optimal fair algorithm also has to pull each sub-optimal arm $i$ for $\floor{r_iT} - \alpha$ rounds. 
A learning algorithm that pulls a sub-optimal arm $i$ for more than $\floor{r_iT} - \alpha$ rounds, incurs a regret of $\Delta_i = \mu_1 - \mu_i$ for each extra pull. The technical difficulties in designing an optimal algorithm for the \FSMAB problem are the conflicting constraints on the quantity $N_{i,T} - \floor{r_iT}$ for a sub-optimal arm $i \neq 1$: at any time $T$ for the algorithm to be fair we want $N_{i,T} - \floor{r_iT}$ to be at least $\alpha$ whereas to minimize the regret we want $N_{i,T} - \floor{r_iT}$ to be close to $\alpha$.
% An algorithm that is not aware of the true means of the reward distributions of arms, faces the exploration v/s exploitation dilemma. On one hand, it has to sufficiently explore all the arms so as to find an optimal arm and on the other, it must exploit the information gathered about mean rewards of the arms. The fairness constraints assist in exploration by guaranteeing $\floor{r_iT} - \alpha$ samples for each arm $i$.

%% file: framework.tex
\section{A Framework for \FSMAB Algorithms}
\label{sec: proposed algorithms}

%\subsection{\FALG: Class of \FSMAB Algorithms}
\label{subsec: fair algorithm class}
\noindent In this section, we provide the framework of our proposed class of \FSMAB algorithms. Our meta-algorithm \FALG is given in Algorithm \ref{alg:fair Alg}. The key result in this work is the following theorem, which guarantees that \FALG is \aFair (see Definition \ref{definition: fairness}), independent of the choice of the learning algorithm \learn($\cdot$). Note that the fairness guarantee holds uniformly over the time horizon, for any sequence of arm pulls by \falg.%In each round $t$ we're interested in the arms that could possibly violate the fairness constraints and hence look at arms $i \in [k]$ such that $\alpha < r_i(t-1) - N_{i,t-1} < \alpha + 1$. Having provided this intuition, we describe our algorithm.
%, which is equivalent to $r_it - N_{i,t} < \alpha + 1$
%Recall from Section \ref{subsec:model fairness}, an algorithm is said to be \aFair, if it satisfies $\floor{r_it} - N_{i,t} \leq \alpha$, for all $t \leq T$, for all arms $i \in [k]$. 

\begin{algorithm}[ht!]
 \SetAlgoLined
 \KwIn{$[k], (r_i)_{i \in [k]}, \alpha \geq 0$, \learn($\cdot$) }
 \textbf{Initialize:} \\ 
   $N_{i,0} = 0$ for all $i \in [k]$\\
   $S_{i,0} = 0$ for all $i \in [k]$, where $S_{i,t} = $ total reward of arm $i$ in $t$ rounds\\
  \For{$t = 1,2, \ldots$} { 
     Define : $A(t) = \Big\{ i ~\big|~ r_i \cdot (t-1) - N_{i,t-1} > \alpha \Big\}$ \\ 
     %Define : $S_t = \Big\{ i ~\big|~ r_i \cdot (t-1) - N_{i,t-1} \leq \alpha \Big\}$ \\ 
     Pull arm $i_t = \begin{cases}
    \argmax_{i \in [k]} \big(r_i \cdot (t-1) - N_{i,t-1}\big) & \text{If} A(t) \neq \emptyset  \\ 
    \learn(N_t,S_t) & \text{Otherwise}  \\ 
    \end{cases}$ \\ 
    Update parameters $N_t$ and $S_t$
    }
  \caption{\falg}
  \label{alg:fair Alg}
\end{algorithm}
\begin{restatable}{theorem}{FAlgFair}\label{theorem: fairness of fAlg}
For a given $\alpha \geq 0$ and for any given fairness constraint vector $r = (r_1, r_2, \ldots, r_k)$ where $r_i \in [0, \frac{1}{k})$ for all $i \in [k]$, \FALG is \aFair irrespective of the choice of the learning algorithm \learn($\cdot$).
\end{restatable}
The proof of Theorem \ref{theorem: fairness of fAlg} is given in Section \ref{sec: theoretical results}. The guarantee in the above theorem also holds when $\alpha = 0$ and hence \FALG with $\alpha = 0$ is \emph{fair}. In particular, when the learning algorithm \learn($\cdot$) = \ucb1, we call this algorithm \fucb. We provide the \rRegret bound for \fucb.
\begin{restatable}{theorem}{FairUCBRegret}
\label{theorem: FUCB regret}
%For \FSMAB problem, \FUCB has \rRegret $\mathbb{E}[\mathcal{R}_{\fucb}^r (T)] = O(\sum_{i \neq 1}\frac{\ln T}{\Delta_i})$.
The \rRegret of \FUCB is given by
\begin{align*}
\small
    \mathcal{R}_{\fucb}^{r}(T) &\leq \Big(1 + \frac{\pi^2}{3} \Big) \cdot \sum_{i \in [k]}\Delta_i  \\ & +  \sum_{\substack{i \in S(T) \\ i \neq 1}}\Delta_i \cdot \bigg( \frac{8 \ln T}{\Delta_i^2} - \Big(r_i \cdot T - \alpha \Big) \bigg) 
\end{align*}
where $S(T) = \Big\{i \in [k] ~\big| ~r_i \cdot T - \alpha < \frac{8 \ln T}{\Delta_i^2}\Big\}$. In particular for large enough $T$, $\mathcal{R}_{\fucb}^{r}(T) \leq \Big(1 + \frac{\pi^2}{3} \Big) \cdot \sum_{i \in [k]}\Delta_i$.
\end{restatable}
Theorem \ref{theorem: FUCB regret} is proved in Section \ref{sec: theoretical results}. Observe that if $S(T)$  $\neq$ $\emptyset$ the \rRegret of \FUCB is sub-logarithmic and if $S(T)$ $ = $ $\emptyset$ then the \rRegret is constant. We prove the distribution-free regret of \FUCB in Theorem \ref{thm: instance independent regret} (proof in Appendix (Section \ref{sec:omitted})).
%\footnote{The proofs of Theorems 3 and 4 can be found in the extended version of this paper.}.
\begin{restatable}{theorem}{InstanceIndependentRegret}\label{theorem: distribution independent regret}
\label{thm: instance independent regret}
The distribution-free \rRegret of \FUCB is $O(\sqrt{T\ln T})$.
\end{restatable}
We conclude this section by observing that as the fairness guarantees of \FALG hold without any loss in  \textsc{Learn($\cdot$)}, this framework can easily be made operational in practice.

%% file: cost_of_fairness.tex
\section{Cost of Fairness}
\label{sec:cost_of_fair}
Our regret guarantees until now have been in terms of \rregret, but now we evaluate the \emph{cost of fairness} in terms of the conventional notion of regret. In particular, we show the trade-off between the conventional regret and fairness in terms of the \TOL.
\begin{restatable}{theorem}{thmAlphaRegret}
\label{theorem:alpha regret}
The expected regret of \FUCB is given by
\begin{align*}
    \mathcal{R}(T) \leq& \sum_{i \in S(T)} (r_i\cdot T - \alpha) \cdot \Delta_i + \sum_{\substack{i \in S(T) \\ i \neq 1}}  8\ln T/ \Delta_i\\ & 
    + \sum_{i \in [k]} (1 + \pi^2/3)\cdot \Delta_i \tag{where $S(T) = \{i ~|~ (r_i \cdot T - \alpha) < 8\ln T/\Delta_{i}^2\}$.}
\end{align*}

\end{restatable}

%The proof of Theorem \ref{theorem:alpha regret} is given in Appendix \ref{sec:omitted}.  
Theorem \ref{theorem:alpha regret} captures the explicit trade-off between regret and fairness in terms of the \TOL parameter $\alpha$. If $S(T) = \emptyset$ we have that the regret is $O(\ln T)$. This implies that if $\alpha > r_{i}T - 8 \ln T/\Delta_i^2$ for all $i\neq 1$, then the regret is $O(\ln T)$. However, if $S(T) \neq \emptyset$ then for each $i \in S(T)$, an additional regret equal to $r_{i}T - \alpha$ is incurred in which case the regret is $O(T)$. We complement these results with simulations in Section \ref{sec:simulation}.
%Notice the trade-off between fairness guarantees achieved by the algorithm and the asymptotic regret guarantees. 

%% file: proofs.tex
\section{Proof of Theoretical Results}
\label{sec: theoretical results}
%\noindent We begin by first analyzing the fairness guarantee provided by \FALG.
\noindent\textbf{Proof of Theorem \ref{theorem: fairness of fAlg}}\\
After each round $t$ (and before round $t+1$), we consider the $k+1$ sets, $M_{1,t}, M_{2,t},\ldots,M_{k,t}$, and $S_t$, as defined below:\begin{itemize}[noitemsep]
    \item  arm $i \in M_{j,t} \iff \alpha + \frac{(k - j)}{k} \leq r_{i}t - N_{i,t} < \alpha + \frac{(k - j + 1)}{k}$, $\forall j \in [k]$
    \item  arm $i \in S_t \iff r_{i}t - N_{i,t} <  \alpha$
\end{itemize} 
\begin{figure}[ht!]
\begin{center}
\input{figure_partition_set.tex}
\end{center}
\caption{Partition of the arms}\label{figure :partition of the arms}
\end{figure}
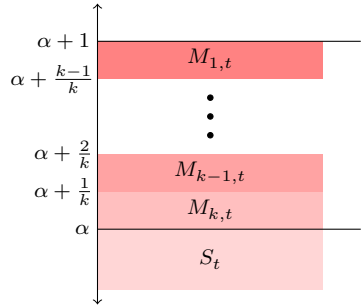
Let $V_{j,t} = \uplus_{\ell = 1}^j M_{\ell,t}$, for all $j \in [k]$. Then the following lemma guarantees the fairness of the algorithm and is at the heart of the proof. The proof of the theorem is immediate from the proof of the  lemma.
\begin{restatable}{lemma}{lemFair}
\label{lem:FairnessLemma}
For $t \geq 1$, we have
\begin{enumerate}
    \item $V_{k,t} \uplus S_t = [k]$
    \item $|V_{j,t}| \leq j$, for all $j \in [k]$
\end{enumerate}
\end{restatable}
\noindent Condition 1 in Lemma \ref{lem:FairnessLemma} ensures that at any time $t \geq 1$, the $k+1$ sets $M_{1,t}, M_{2,t},\ldots,M_{k,t}, S_t$ form a partition of the set $[k]$ of arms. Hence the arm pulled at the $(t+1)$-th round by the algorithm is from one of these $k+1$ sets. As a part of the proof of Lemma \ref{lem:FairnessLemma}, in Observation \ref{obs:armPulled} we show that if $i_{t+1}$ is the arm pulled at the ($t+1$)-th round then after $t+1$ rounds $i_{t+1} \in M_{k,t+1} \uplus S_{t+1}$. Also in Observation \ref{obs:armNotPulled} we show that if an arm $i \in M_{j,t}$ is not pulled in the ($t+1$)-th round then after $t+1$ rounds arm $i \in M_{j-1,t+1} \uplus M_{j,t+1}$ for all $j \in [2,k]$. We note that the two conditions in Lemma \ref{lem:FairnessLemma} are true after the first round, and then the Observations \ref{obs:armPulled} and \ref{obs:armNotPulled} together ensure that these conditions remain true for all $t>1$. Hence, all arms $i \in [k]$ satisfy $r_{i}t - N_{i,t} < \alpha + 1$ for all $t \geq 1$, which implies $\floor{r_{i}t} - N_{i,t} \leq \alpha$. In particular, we have $\floor{r_{i}t} - N_{i,t} \leq \alpha$, for all $t \geq 1$, for all $i \in [k]$, which by Definition \ref{definition: fairness} proves that \FALG is \aFair. \hfill $\Box$ 

\emph{Proof of Lemma \ref{lem:FairnessLemma}}: 
We begin with two complementary observations and then prove the lemma by induction.
\begin{observation}
\label{obs:armPulled}
Let $i$ be the arm pulled by \FALG in round $t+1$. 

\noindent 1. if $i \in S_t$, then $i \in S_{t+1}$ 

\noindent 2. if $i \in M_{j,t}$ for some $j \in [k]$, then $i \in M_{k,t+1} \uplus S_{t+1}$
\end{observation}
\begin{proof} \phantom{\qedhere}
\textit{Case 1: } $i \in S_t \implies r_{i}t - N_{i,t} < \alpha$. Then after round $t+1$, we have
\begin{align*}
    r_{i}(t+1) - N_{i,t+1} &= r_{i}t + r_{i} - N_{i,t} - 1\\
    & < \alpha - (1 - r_{i})\\
    & < \alpha \tag{Since $1 - r_{i} > 0$}
\end{align*}
$$\implies i \in S_{t+1}$$
\textit{Case 2: } $i \in M_{j,t}$ for some $j \in [k] \implies r_{i}t - N_{i,t} < \alpha + \frac{(k - j + 1)}{k}$. Then after round $t+1$, we have
\begin{align*}
    r_{i}(t+1) - N_{i,t+1} &= r_{i}t + r_{i} - N_{i,t} - 1\\
    & < \alpha + \frac{(k - j + 1)}{k} - (1-r_{i})\\
    & < \alpha -\frac{j}{k}+\frac{1}{k} +r_{i} \\
    & < \alpha + r_{i} ~~<~~ \alpha + \frac{1}{k} \tag{Since $r_{i} < \frac{1}{k}$} \\ %\tag{Since $j \geq 1$}\\
    \implies i &\in M_{k.t+1}\uplus S_{t+1} \tag*{\qed}
  %  & < \alpha + \frac{1}{k} \tag{Since $r_{i} < \frac{1}{k}$}
\end{align*}
\end{proof}

\begin{observation}
\label{obs:armNotPulled}
Let $i \in [k]$ be any arm not pulled at time $t+1$.

\noindent 1. If $i \in S_t$, then $i \in S_{t+1} \uplus M_{k,t+1}$

\noindent 2. If $i \in M_{j,t}$ for $j \in [2,k]$, then $i \in M_{j-1,t+1}\uplus M_{j,t+1}$ 
\end{observation}
\begin{proof}
\textit{Case 1: } $i \in S_t \implies r_{i}t - N_{i,t} < \alpha$. Then after round $t+1$, we have
\begin{align*}
    r_{i}(t+1) - N_{i,t+1} &= r_{i}t - N_{i,t} + r_{i} \tag{  $N_{i,t+1} = N_{i,t}$}\\
    & < \alpha + r_{i} ~~<~~ \alpha + \frac{1}{k} \tag{Since $r_{i} < \frac{1}{k}$} \\
    \implies & i \in S_{t+1} \uplus M_{k,t+1}
    %& < \alpha + \frac{1}{k} \tag{Since $r_{i} < \frac{1}{k}$}
\end{align*}
%$$\implies i \in S_{t+1} \uplus M_{k,t+1}$$
\textit{Case 2: } $i \in M_{j,t}$ for some $j \in [2,k] \implies \alpha + \frac{k-j}{k} \leq r_{i}t - N_{i,t} < \alpha + \frac{(k - j +1)}{k}$. Then after round $t+1$, we have
\begin{align*}
    r_i(t+1) - N_{i,t+1} &= r_it - N_{i,t} + r_i \\%\tag{Since, $N_{i,t+1} = N_{i,t}$}\\
    & < \alpha + \frac{(k - j + 1)}{k} + r_i\\
    & < \alpha + \frac{(k - j + 1)}{k} + \frac{1}{k} \\%\tag{Since $r_i < \frac{1}{k}$}\\
    & = \alpha + \frac{(k - (j - 1) + 1)}{k}
\end{align*}
and~~$ r_{i}t - N_{i,t} + r_{i} \geq \alpha + \frac{k-j}{k} + r_{i} ~~  \geq ~~\alpha + \frac{k-j}{k}$ \vspace{0.1in}

   % & \geq \alpha + \frac{k-j}{k} \tag{Since $r_{i} \in [0,1/k)$} \\ 
\textcolor{white}{0}~~~~~~~~~~~~~~$\implies i  \in M_{j-1,t+1}\uplus M_{j,t+1}$
\end{proof}
With the above observation we complete the proof of the lemma using induction. 

\noindent\underline{Induction base case ($t = 1$)}: Let $i_1$ be the arm pulled at $t=1$. Then 
$$r_{i_1}t - N_{i_1,1} = r_{i_1} - 1 < 0 \leq \alpha$$
$$\implies i_1 \in S_1$$
For all $i \neq i_1$, we have $r_{i}t - N_{i,1} = r_{i} < \frac{1}{k} \leq \alpha + \frac{1}{k} \implies i \in S_1 \uplus M_{k,1}$. Hence, $ V_{k,1} \uplus S_1 = [k]$, $|V_{k,1}| \leq k-1$, and $|V_{j,1}| = 0$ for all $j \in [k-1]$.
%\begin{equation*}
 %   V_{k,1} \uplus S_1 = [k], ~ |V_{k,1}| \leq k-1, \textnormal{ and } ~|V_{j,1}| = 0 \textnormal{ for all } j \in [k-1]
%\end{equation*}
Thus, conditions (1) and (2) of the lemma hold.

\noindent\underline{Inductive Step}: Assuming the conditions in the lemma hold after round $t$, we show that they hold after round $t+1$.

\noindent \textit{Case 1: } $i_{t+1} \in S_t$. From Observation \ref{obs:armPulled}, we know $i_{t+1} \in S_{t+1}$. From Observation \ref{obs:armNotPulled}, we know that for any arm $i \neq i_{t+1}$, $i \in S_{t+1} \uplus M_{k,t+1}$. Hence, $V_{k,t+1} \uplus S_{t+1} = [k]$, $|V_{j,t+1}| = 0$ for all $j \in [k-1]$, and $|V_{k,t+1}| \leq k - 1$.
%\begin{align*}
%    V_{k,t+1} \uplus S_{t+1} &= [k]\\
%    |V_{j,t+1}| &= 0 \text{\hspace{0.5cm} for all $j \in [k-1]$}\\
%    |V_{k,t+1}| &\leq k - 1 < k
%\end{align*}
Thus, Conditions (1) and (2) in the lemma hold after round $t+1$.\\
\textit{Case 2: } $i_{t+1} \in M_{a,t}$, for some $a \in [k]$. 
\begin{align*}
    i_{t+1} &\in M_{a,t} \implies i_{t+1} \in V_{a,t}\\
    %\implies i_{t+1} &\in V_{a,t}\\
    \implies |V_{j,t}| &= 0 \text{\hspace{0.3cm} for all $j \in [1,a-1]$ if $a>1$} 
\end{align*}
From Observation \ref{obs:armPulled}, we know $i_{t+1} \in S_{t+1} \uplus M_{k,t+1}$, and from Observation \ref{obs:armNotPulled}, we infer that $V_{j-1,t+1} = V_{j,t}\setminus \{i_{t+1}\}$ for all $j\in [2,k$]. Also,
\begin{align*}
    |V_{j,t} \setminus \{i_{t+1}\}| &\leq j-1 \text{\hspace{0.5cm} for all $j \in [a,k]$}\\
    \implies |V_{j,t+1}| &\leq j \text{\hspace{0.5cm} for all $j \in [k]$}
\end{align*}
Also, $V_{k,t+1} \uplus S_{t+1} = [k]$. Hence, Conditions (1) and (2) of the lemma hold after round $t+1$. \hfill $\Box$ 

\noindent\textbf{Proof of Theorem \ref{theorem: FUCB regret}}\\
The regret analysis of \FUCB builds on the regret analysis of \ucb1 which we give in the Appendix \ref{subsec:ucbProof}. In Appendix \ref{subsec:ucbProof} we also introduce the notations used in this proof. The \ucb1 estimate of the mean of arm $i$ denoted as $\Bar{\mu}_{i}(t) = \hat{\mu}_{i,N_{i,t-1}}(t-1) + c_{t,N_{i,t-1}}$, where $\hat{\mu}_{i,N_{i,t-1}}(t-1)$ is the empirical estimate of the mean of arm $i$ when it is pulled $N_{i,t-1}$ times in $t-1$ rounds and $c_{t,N_{i,t-1}} = \sqrt{\frac{2\ln t}{N_{i,t-1}}}$ is the confidence interval of the arm $i$ at round $t$. Similar to the analysis of the \ucb1 algorithm, we upper bound the expected number of times a sub-optimal arm is pulled. We do this by considering two cases dependent on the number of times the sub-optimal arm is required to be pulled for satisfying its fairness constraint.

\noindent \underline{Case 1: } Let $i \neq 1$ and $r_i \cdot T - \alpha \geq \frac{8 \ln T}{\Delta_i^2}$. Then 
\begin{align*}
    \mathbb{E}[N_{i,T}] \leq& \big(r_i \cdot T -\alpha\big) + \sum_{t = 1}^T \mathbbm{1}\{i_t = i, N_{i,t-1} \geq r_i \cdot T - \alpha\}\\
    \leq& \big(r_i \cdot T -\alpha\big)  \\   + \sum_{t=1}^\infty  \sum_{s_1=1}^t & \sum_{s_i=r_i \cdot T -\alpha}^t \mathbbm{1}{\Big\{ \hat{\mu}_{1,s_1}(t)  + c_{t,s_1} \leq \hat{\mu}_{1,s_i}(t) + c_{t,s_i} \Big\}} \tag{Follows from Appendix A, Theorem 6}
\end{align*}
Since $r_i \cdot T -\alpha \geq \frac{8 \ln T}{\Delta_i^2}$, it follows from the proof of Theorem 6 in Appendix A that $\mathbb{E}[N_{i,T}] \leq r_i \cdot T - \alpha + \Big( 1 + \frac{\pi^2}{3} \Big)$. Hence, $\mathbb{E}[N_{i,T}] - \big(r_i \cdot T -\alpha\big) \leq \Big( 1 + \frac{\pi^2}{3} \Big)$. 

\noindent\underline{Case 2: }Let $i \neq 1$ and $r_i \cdot T < \frac{8 \ln T}{\Delta_i^2}$\\
Then the proof of Theorem 6 in Appendix A can be appropriately adapted to show that $\mathbb{E}[N_{i,T}] \leq \frac{8 \ln T}{\Delta_i^2} + \Big( 1 + \frac{\pi^2}{3} \Big)$. Hence
$$
    \mathbb{E}[N_{i,T}] - \big(r_i \cdot T - \alpha \big) 
    \leq \frac{8 \ln T}{\Delta_i^2} + \Big( 1 + \frac{\pi^2}{3} \Big) - \big(r_i \cdot T - \alpha \big)
$$
Suppose $S(T) = \Big\{i \in [k] ~\big|~ r_i \cdot T - \alpha < \frac{8 \ln T}{\Delta_i^2}\Big\}$. Then from the two cases discussed above, we can conclude that
\begin{align*}
    \mathcal{R}_{\fucb}^{r}(T) \leq& \Big(1 + \frac{\pi^2}{3} \Big) \cdot \sum_{i \in [k]}\Delta_i  \\  & +  \sum_{i \in S(T),  i \neq 1}   \Delta_i \cdot \bigg( \frac{8 \ln T}{\Delta_i^2} - \big(r_i \cdot T -\alpha \big)\bigg)  
\end{align*}

Hence, $\mathcal{R}_{\fucb}^{r}(T) = O(\sum_{i \neq 1}\frac{\ln T}{\Delta_i})$. \hfill $\Box$

%% file: figure_partition_set.tex
\begin{tikzpicture}
%corners
\coordinate (origin) at (0,0);
\coordinate (axisy1) at (0,3);
\coordinate (axisy2) at (0,-1);
\coordinate (axisx) at (3.5,0);
\coordinate (a2) at ($(origin) + (3,0.5)$);
\coordinate (b1) at ($(origin) + (0,0.5)$);
\coordinate (b2) at ($(b1) + (3,0.5)$);
\coordinate (c1) at ($(b1) + (0,1.5)$);
\coordinate (c2) at ($(c1) + (3,0.5)$);
\coordinate (d1) at ($(origin) + (0,2.5)$);

%label: nodes
\node at ($ (origin) + (-0.20,0)$) [fill=white!100!] {\scriptsize $\alpha$};
\node at ($ (origin) + (-0.45,0.5)$) [fill=white!100!] {\scriptsize $\alpha + \frac{1}{k}$};
\node at ($ (origin) + (-0.45,1)$) [fill=white!100!] {\scriptsize $\alpha + \frac{2}{k}$};
\node at ($ (origin) + (-0.60,2)$) [fill=white!100!] {\scriptsize $\alpha + \frac{k-1}{k}$};
\node at ($ (origin) + (-0.45,2.5)$) [fill=white!100!] {\scriptsize $\alpha + 1$};

\node at ($ (origin) + (1.5,0.25)$) [fill=white!100!] {\scriptsize $M_{k,t}$};
\node at ($ (b1) + (1.5,0.25)$) [fill=white!100!] {\scriptsize $M_{k-1,t}$};
\node at ($ (c1) + (1.5,0.25)$) [fill=white!100!] {\scriptsize $M_{1,t}$};

\fill [white!50!red!50!] (origin) rectangle (a2);
\fill [white!40!red!60!] (b1) rectangle (b2);
\fill [white!30!red!70!] (c1) rectangle (c2);
\fill [white!60!red!40!] (origin) rectangle (3,-0.8);

\node at ($ (origin) + (1.5,0.25)$)  {\scriptsize $M_{k,t}$};
\node at ($ (b1) + (1.5,0.25)$) {\scriptsize $M_{k-1,t}$};
\node at ($ (c1) + (1.5,0.25)$) {\scriptsize $M_{1,t}$};
\node at ($ (origin) + (1.5,-0.4)$) {\scriptsize $S_t$};

%axis lines
\draw [<->] (axisy1) -- (axisy2);
\draw [-] (origin) -- (axisx);
\draw [-] (d1) -- ($(d1) + (3.5,0)$);

%lines 
\draw [white!50!red!50!] ($(origin) + (3,0)$);
\draw [white!40!red!60!] ($(b1) + (3,0)$);
\draw [white!30!red!70!] ($(c1) + (3,0)$);

%dots
\draw[black,fill=black] (1.5,1.25) circle (.2ex);
\draw[black,fill=black] (1.5,1.5) circle (.2ex);
\draw[black,fill=black] (1.5,1.75) circle (.2ex);
\end{tikzpicture} 

%% file: simulation.tex
\section{Experimental Results}
\label{sec:simulation}
In this section we show the results of simulations that validate our theoretical findings. First, we represent the cost of fairness by showing the trade-off between regret and fairness with respect to the \TOL $\alpha$. Second, we evaluate the performance of our algorithms in terms of \rRegret and fairness guarantee by comparing them with the algorithm by \cite{LLJ19}, called Learning with Fairness Guarantee(\lfg), as a baseline. Note that in Figure 3, cumulative regret is plotted on a logarithmic scale. The rationale behind the choice of instance parameters is discussed in Appendix \ref{app: simulation param}.

\noindent\textbf{Trade-off: Fairness vs. Regret}\\
%\label{sec:cost of fairness}
\noindent We consider the following \FSMAB instance: $k = 10$, $\mu_1 = 0.8$, and $\mu_i = \mu_1 - \Delta_i$, where $\Delta_i = 0.01i$, and $r = (0.05, 0.05, \ldots , 0.05) \in [0,1]^k$. We show the results for $T = 10^6$. Figure 2 shows the trade-off between regret in terms of the conventional regret and maximum fairness violation equal to $\textnormal{max}_{i \in [k]} r_it -N_{i,t}$, with respect to $\alpha$, and this in particular captures the \emph{cost of fairness}. As can be seen, the regret decreases, and maximum fairness violation increases respectively as $\alpha$ increases till a threshold for $\alpha$ is reached. For values of $\alpha$ less than this threshold the fairness constraints cause some sub-optimal arms to be pulled more than the number of times required to determine its mean reward with sufficient confidence. On the other hand, for values of $\alpha$ more than this threshold, the regret reduces drastically, and we recover logarithmic regret as could be expected from the classical \ucb1 algorithm. Note that the threshold for $\alpha$ in this case is problem-dependent.

\begin{figure*}[ht!]
\begin{tabular}{ccc}
\hspace{-8mm}
\includegraphics[scale=0.25]{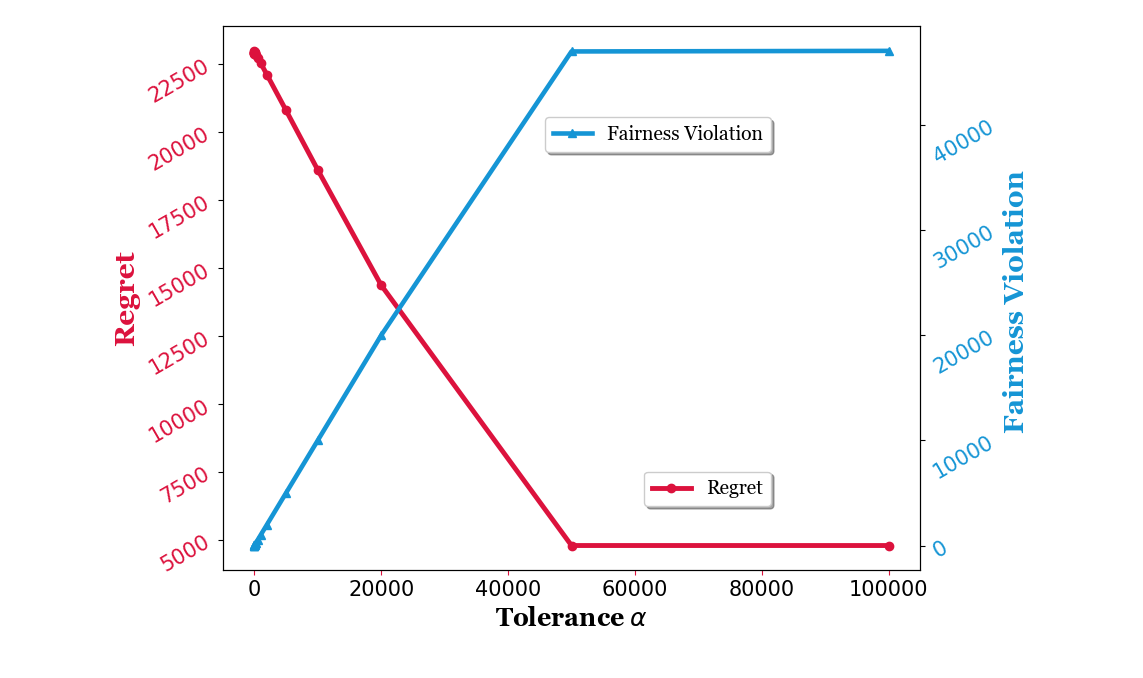}
     &  
     \hspace{-8mm}
     \includegraphics[scale=0.35]{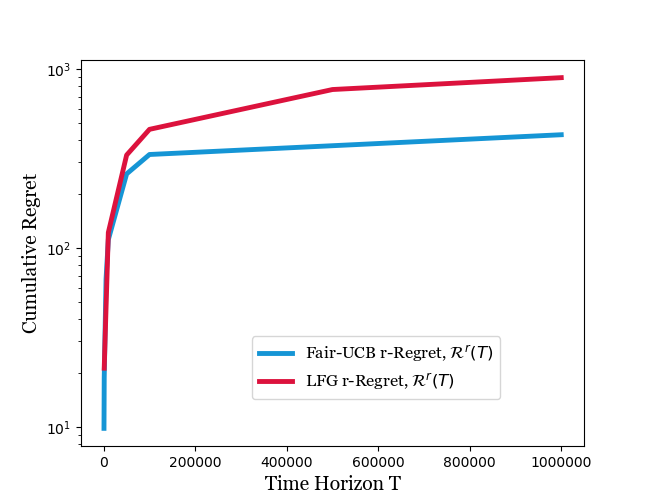}
     & 
     \hspace{-8mm}
     \includegraphics[scale=0.25]{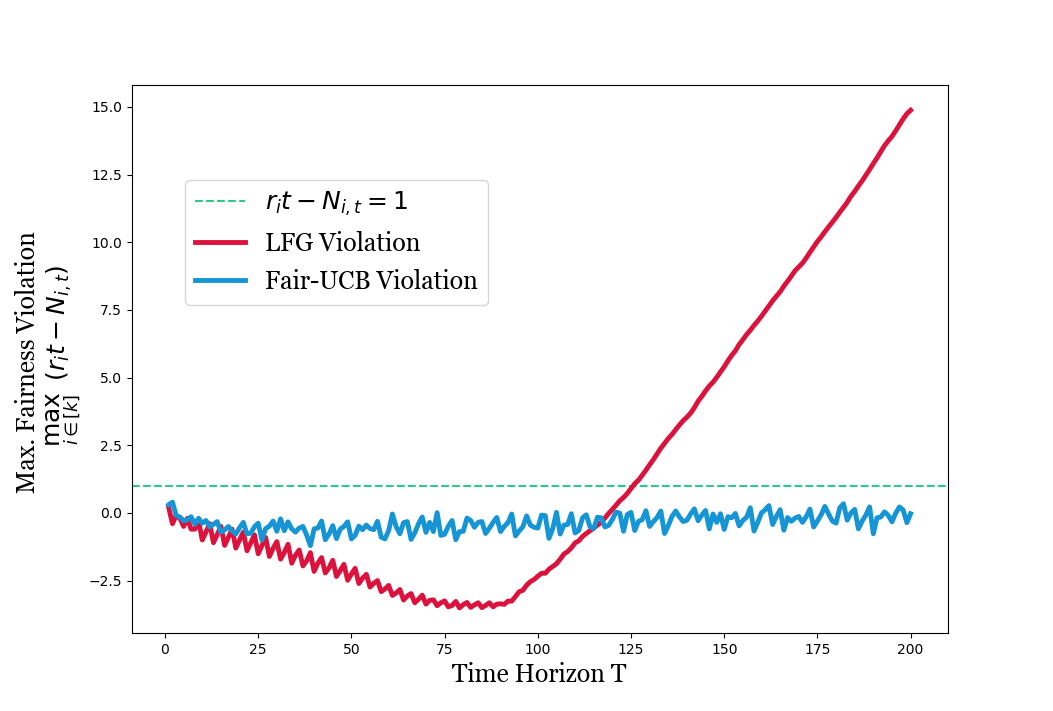}
     \\
     \textbf{Figure 2}: Cost of Fairness
     &
     \textbf{Figure 3}: \rRegret Guarantee
     &
     \textbf{Figure 4}: Fairness Guarantee
\end{tabular}
    \centering
    %\caption{}
    \label{fig:my_label}
\end{figure*}

%\begin{figure}[ht!]
%        \includegraphics[width=1.1\linewidth]{CostOfFairness_10pow6.png}
%        \caption{\footnotesize Trade-off: Regret vs. Fairness}
%        \label{fig:cost_of_fairness}
%        \end{figure}
%    \begin{figure}[ht!]
%        \includegraphics[width=1.1\linewidth]{rRegret_FUCBvsLFG.png}
%        \caption{\footnotesize \rregret: \FUCB vs. \lfg}
%        \label{fig:rRegret Comparison}
%\end{figure}
%\begin{figure}[ht!]
%        \includegraphics[width=1.1\linewidth]{fairness_violation.png}
 %       \caption{\footnotesize \rregret: \FUCB vs. \lfg}
 %       \label{fig:fairness violation}
%\end{figure}

\noindent\textbf{Comparison: \FUCB vs. \lfg:}
%\label{sec:comparison fucb vs lfg}
%As we detailed in Section \ref{sec:relWork}, t
The work closest to ours is the one by \cite{LLJ19} and their algorithm, which is called \emph{Learning with Fairness Guarantee} (\lfg), is used as a baseline in the following simulation results. The simulation parameters that we consider for comparing \rRegret are the same as in the previous instance. Figure 3 shows the plot of time vs. \rRegret for \FUCB and \lfg. Note that \FUCB and \LFG perform comparably in terms of the \rRegret suffered by the algorithm. Also, the simulation results validate our theoretical result of logarithmic \rRegret bound.

We next compare fairness guarantee of \FUCB with that of \lfg. We consider an instance with $k = 3$, $\mu = (0.7,0.5,0.4)$, $r = (0.2,0.3,$ $0.25)$ and, $\alpha = 0$. %Even though we tested the fairness guarantee for $T=10^6$, we show the plot for $T=200$ as, for this instance, it is the appropriate scale to compare the performance of \FALG and \LFG without losing any details in terms of the fairness violation. 
Figure 4 shows the plot of time vs. maximum fairness violation. Observe that the fairness guarantee of \FUCB holds uniformly over the time horizon $T$. Note that, though the fairness violation for \LFG appears to be increasing, it does reduce at some point and go to zero which guarantees asymptotic fairness. To summarize, the simulation result reaffirm our theoretical guarantees for both fairness and \rRegret of \FALG in general, and \FUCB in particular.

%% file: related_work.tex
\section{Related Work}
\label{sec:relWork}
%\todo[inline]{ TODO \\ 1. List of the related literature on fairness in MAB's (David Parkes and Debmalya had a paper some time back on this ...) \\ 
%2. A single paragraph summary of most relevant papers (three to five)}.
There has been a surge in research efforts aimed at ensuring fairness in  decision making by machine learning algorithms such as classification algorithms \citep{AGA18,NAR18,ZAF17a,ZAF17}, regression algorithms \citep{BERK17,REZ19}, ranking and recommendation systems \citep{SIN19AXRV,BEU19AXRV,SIN18,CEL17,ZEH17}, etc. This is true even in the context of online learning, particularly in the \SMAB setting. We state these relevant works below.

\citet{JKMR16} propose a variant of the $\UCB$ algorithm that ensures what they call meritocratic fairness i.e. an arm is never preferred over a better arm irrespective of the algorithm's confidence over the mean reward of each arm. This guarantees individual fairness (see \citet{DWORK12}) for each arm while achieving efficiency in terms of sub-linear regret. %In contrast, we consider that the fairness constraints are exogenously specified and the choices made by the algorithm must adapt to these constraints so as to minimize the regret while satisfying these constraints. 
The work by \citet{LRDMP17} aims at ensuring ``treatment equality", wherein similar individuals are treated similarly. \citet{GCKR18} consider individual fairness guarantees with respect to an unknown fairness metric.%This outcome based notion of fairness considers that the fairness constraints are built into the problem. 

The papers discussed above combine the conventional goal of maximizing cumulative reward with that of simultaneously satisfying some additional constraints. \SMAB problem with other added constraints have been considered. For example, \citet{BAD13,immorlica2018adversarial} study the \SMAB with knapsack constraints, where the number of times that a particular arm can be pulled is limited by some budget. The works of \citet{XIA15,amin2012budget,tran2014efficient} consider the \SMAB problem in which there is some cost associated with pulling each arm, and the learner has a fixed budget. The work by \citet{LATTIMORE2014,LATTIMORE15,TAL18} investigates bandit optimization problems with resource allocation constraints.
%Various other constraints such as bandits with sleeping arms, etc. have also been well-studied in literature \cite{KLE10,CHA17}.

\textbf{Comparison with \citet{LLJ19}}: 
%The \FSMAB model seamlessly combines the goal of ensuring fairness with that of studying the \SMAB problem with added constraints. A recent paper by \cite{LLJ19} considers a combinatorial (additive rewards), sleeping (stochastic availability) \SMAB setup with fairness constraints similar to the ones considered in this paper, and they propose an algorithm called Learning with Fairness Guarantee (\lfg), which controls the trade-off between minimizing regret and satisfying fairness constraints using a tuning parameter. In our simulations we consider the \LFG algorithm as a baseline to compare the performance of our algorithm in terms of both, fairness and regret. 
In addition to proving a $O(\sqrt{T\ln T})$ distribution-free \rRegret bound as in \citet{LLJ19}, we show a $O(\ln T)$ \rRegret bound with finer dependence on the instance parameters.  Our fairness guarantee holds uniformly over time and hence is much stronger than the asymptotic fairness guarantee of \lfg. Moreover, as our fairness guarantee is independent of the learning algorithm used in \falg, it holds for the setting considered in \cite{LLJ19}.% and the \rRegret in this case is determined by the regret of the corresponding learning algorithm.

\textbf{Comparison with \citet{CELIS2018}}: A recent work by \citet{CELIS2018} considers a personalized news feed setting, where at any time $t$, for a given context (user), the arm (i.e. ad to be displayed) is sampled from a distribution $p_t$ over the set $[k]$ of arms (ads) and fairness is achieved by ensuring a pre-specified probability mass on each arm which restricts the allowable set of distributions to a subset of the simplex. The algorithm in \cite{CELIS2018} when applied the classical stochastic multi-armed bandit setting considered by us, ensures any-time fairness only in \emph{expectation} over the random pulls of arms by the algorithm. In contrast, our algorithm (Theorem \ref{theorem: fairness of fAlg}) provides much stronger deterministic any-time fairness guarantee. Further, we also provide an explicit trade-off (in terms of the unfairness tolerance $\alpha$), between fairness and regret. Also, the computational overhead of our algorithm is just $O(1)$, whereas the algorithms in \cite{CELIS2018} need to solve at least one linear program in each round. We also note that our model can directly be adapted to capture the setting in \cite{CELIS2018}.
%as follows: for a fixed user (context), our algorithm must choose each arm (ad) a pre-specified fraction of the rounds, at any time $t$. 
%In particular, an instance of FAIR-LEARN is run independently for each user (context).  
%Note that this ensures any-time fairness only in \emph{expectation} over the random pulls of arms by the algorithm. In contrast, we consider fairness in the classical stochastic multi-armed bandit setting which can directly be adapted to capture their setting as follows: for a fixed user (context), our algorithm must choose each arm (ad) a pre-specified fraction of the rounds, at any time $t$. In particular, an instance of FAIR-LEARN is run independently for each user (context). Notably, for a given context, our algorithm (Theorem \ref{theorem: fairness of fAlg}) provides a deterministic any-time fairness guarantee, which is much stronger than the fairness guarantee in Celis et al, and simultaneously achieves the optimal $r$-regret guarantee. Further, we also provide an explicit trade-off (in terms of the unfairness tolerance $\alpha$), between fairness and regret. Moreover, the computational overhead of our algorithm is just $O(1)$, whereas their algorithms need to solve LPs in each round. 
%\begin{itemize}
%    \item \cite{JKMR16}
%    \item \cite{LLJ19}
%    \item \cite{LRDMP17}
%    \item \cite{GCKR18}
%    The most relevant work to ours is \cite{LLJ19}. 
%\end{itemize}

%% file: discussion.tex
\section{Discussion and Future Work}
\label{sec:discussion}
The constraints considered in this paper capture fairness by guaranteeing a minimum fraction of pulls to each arm at all times. There are many situations where such fairness constraints are indispensable, and in such cases the \rRegret notion compares the expected loss of any online algorithm with the expected loss of an optimal algorithm that also satisfies such fairness constraints. An important feature of our proposed meta algorithm \FALG is the uniform time fairness guarantee that it provides independent of the learning algorithm used. We also elucidate the cost of satisfying such fairness constraints by evaluating the trade-off between the conventional regret and fairness in terms of an unfairness tolerance parameter. Several notions of fairness such as disparate impact, statistical parity, equalized odds, etc. have been extensively studied in the machine learning literature (see \cite{barocasfairml}). Incorporating such fairness notions in online learning framework, as done by \citet{BLUM2018,BLUM2019,BECHAVOD2019}, is an exciting future direction.

%% file: appendix.tex
\newpage 
\appendix
\label{appendix}

\section{Preliminaries}
\label{sec:prelims}

\begin{subsection}{Hoeffding's Lemma}
\begin{theorem}
Let $X_1,X_2, \ldots , X_n $ be i.i.d. random variables with $X_i \in [a, b]$ and $\mathbbm{E}[X_i] = \mu $.
Then $$ 
Pr \bigg(~\biggl\vert\frac{1}{n} \sum_{i=1}^{n}X_i - \mu_i \biggr\vert \geq \epsilon \bigg) \leq 2 e^{\frac{-2n\epsilon^2}{(b-a)^2}} 
$$
\end{theorem}
\end{subsection}

\begin{subsection}{Upper Confidence Bound (\ucb) based Algorithm}
\label{subsec:ucbProof}
In this section we describe the \ucb1 algorithm that was introduced by \cite{AUER02} and for completeness we also give a proof of its regret bound. In the \ucb1 algorithm for each arm the algorithm maintains a \ucb1 estimate and at each round the algorithm plays the arm with the highest \ucb1 estimate. Such a \ucb1 estimate for an arm $i \in [k]$ at round $t$ is dependent on the empirical mean of the rewards of arm $i$ and a confidence interval associated with arm $i$. To state it formally let $N_{i,t-1}$ denote the number of times arm $i$ is pulled in $t-1$ rounds. Then the \ucb1 estimate for arm $i \in [k]$ at round $t \geq 1$~ is $\Bar{\mu}_{i}(t) = 0$ if $N_{i,t-1} = 0$, otherwise $\Bar{\mu}_{i}(t) = \hat{\mu}_{i,N_{i,t-1}}(t-1) + \sqrt{\frac{2\ln (t)}{N_{i,t-1}}}$~ where $\hat{\mu}_{i,N_{i,t-1}}(t-1)$ is the empirical mean of the rewards of arm $i$ after being pulled $N_{i,t-1}$ times in $t-1$ rounds and $\sqrt{\frac{2\ln (t)}{N_{i,t-1}}}$ is its associated confidence interval. For ease of notation, we will denote by $c_{t,s_i}$ the confidence interval of arm $i$ at time $t$ when it is pulled $s_i$ times i.e. $c_{t,s_i} = \sqrt{\frac{2\ln (t)}{s_i}}$ . Technically for the first $k$ rounds the algorithm plays each arm once to compute a non-zero \ucb1 estimate for each arm and for every round $t \geq k+1$ it plays the arm with the highest \ucb1 estimate. The total expected regret of \ucb1 after $T$ rounds is given by the following theorem, where $\Delta_i = \mu_1 - \mu_i$ for all $i \in [k]$, and $\Delta_i > 0$ as $\mu_1 > \mu_i$ for $i \neq 1$.
\begin{theorem}\label{theorem: UCB}
For the \SMAB problem, the \ucb1 has expected regret $\mathbb{E}[\mathcal{R}_{\textnormal{\ucb}}(T)] \leq \sum_{i \neq 1} \big(\frac{8 \ln T}{\Delta_i} \big) + \big(1 + \frac{\pi^2}{3}\big) \sum_{i \in [k]} \Delta_i$ .
\end{theorem}
\begin{proof}
To bound the regret of the \ucb1 algorithm, we first upper bound $\mathbb{E}[N_{i,T}]$ for $i \neq 1$, i.e. the expected number of pulls of a sub-optimal arm $i \neq 1$ in $T$ rounds. Denote the arm pulled by the algorithm at the $t$-th round as $i_t$. In the equation below $\mathbbm{1}\{i_{t} = i\}$ is an indicator random variable that is equal to $1$ if $i_t = i$ and is $0$ otherwise. In general $\mathbbm{1}\{\mathsf{E}\}$ denotes an indicator random variable that is equal to $1$ if the event $\mathsf{E}$ is true and is $0$ otherwise.
$$ N_{i,T} = 1 + \sum_{t = k+1}^T \mathbbm{1}\Big\{i_{t} = i\Big\} $$
For any positive integer $\ell$ we may rewrite the above equation as 
\begin{equation}\label{equation: renaming the upper on sub-optimal pulls}
N_{i,T}  \leq \ell + \sum_{t = \ell}^T \mathbbm{1}\Big\{i_{t} = i, N_{i,t-1} \geq \ell\Big\} 
\end{equation}
If $i_t = i$ then $\Bar{\mu}_{1}(t) < \Bar{\mu}_{i}(t)$~ i.e.~ $\hat{\mu}_{1,N_{1,t-1}}(t-1) + c_{t,N_{1,t-1}} < \hat{\mu}_{i,N_{i,t-1}}(t-1) + c_{t,N_{i,t-1}}$~. Hence from Equation \ref{equation: renaming the upper on sub-optimal pulls}
\begin{align*}
\small 
    N_{i,T} & \leq \ell +  \sum_{t = \ell}^T \mathbbm{1}\Big\{ \hat{\mu}_{1,N_{1,t-1}}(t-1) + c_{t,N_{1,t-1}}  \\ & \hspace{60pt}< \hat{\mu}_{i,N_{i,t-1}}(t-1) + c_{t,N_{i,t-1}}~, ~N_{i,t-1} \geq \ell \Big\}\\
    & \leq \ell +\sum_{t = \ell}^T \mathbbm{1}\Big\{ \min_{0<s_1< t} ~\hat{\mu}_{1,s_1}(t-1) + c_{t,s_1} \\ & \hspace{60pt} <  \max_{\ell\leq s_i < t}~\hat{\mu}_{i,s_i}(t-1) + c_{t,s_i} \Big\}\\
    & \leq \ell  + \sum_{t = \ell}^T ~\sum_{s_1=1}^t ~\sum_{s_i=\ell}^t \mathbbm{1}\Big\{ \hat{\mu}_{1,s_1}(t-1) + c_{t,s_1}  \\  & \hspace{60pt} < \hat{\mu}_{i,s_i}(t-1) + c_{t,s_i} \Big\}
\end{align*}
At time $t$, $\hat{\mu}_{1,s_1}(t-1) + c_{t,s_1} < \hat{\mu}_{i,s_i}(t-1) + c_{t,s_i}$ implies that at least one of the following events is true
\begin{align}
    \big\{\hat{\mu}_{1,s_1}(t-1) \leq \mu_1 - c_{t,s_1}\big\} \label{eq:EventA} \\
    \big\{\hat{\mu}_{i,s_i}(t-1) \geq \mu_i + c_{t,s_i}\big\} \label{eq:EventB}\\
    \big\{\mu_1 < \mu_i + 2c_{t,s_i}\big\} \label{eq:EventC}
\end{align}
The probability of the events in Equations \ref{eq:EventA} and \ref{eq:EventB} can be bounded using Hoeffding's inequality as:
\begin{align*}
    \mathbb{P}\Big(\big\{\hat{\mu}_{1,s_1}(t-1) \leq \mu_1 - c_{t,s_1}\big\}\Big) \leq t^{-4}\\
    \mathbb{P}\Big(\big\{\hat{\mu}_{i,s_i}(t-1) \geq \mu_i + c_{t,s_i}\big\}\Big) \leq t^{-4}\\
\end{align*}
The event in equation \ref{eq:EventC}~ $\big\{\mu_1 < \mu_i + 2c_{t,s_i}\big\}$ can be written as $\Big\{\mu_1 - \mu_i - 2\sqrt{\frac{2\ln t}{s_i}} < 0 \Big\}$. Substituting $\Delta_i = \mu_1 - \mu_i$ and if $s_i \geq \Big\lceil{\frac{8\ln T}{\Delta_i^2}\Big\rceil} \geq \Big\lceil{\frac{8\ln t}{\Delta_i^2}\Big\rceil}$ then
\begin{align}
    \mathbb{P}\Bigg( \bigg\{\Delta_i - 2\sqrt{\frac{2\ln t}{s_i}} < 0 \bigg\}   \Bigg) = 0
\end{align}
Thus if $\ell = \frac{8\ln T}{\Delta_i^2}$ then
\begin{align*}
    N_{i,T} & \leq \frac{8\ln T}{\Delta_i^2} + \sum_{t=\frac{8\ln T}{\Delta_i^2}}^T~\sum_{s_1=1}^t~\sum_{s_i=\frac{8\ln T}{\Delta_i^2}}^t  \mathbbm{1}\Big\{ \hat{\mu}_{1,s_1}(t-1) \\ & \hspace{50pt} + c_{t,s_1} < \hat{\mu}_{i,s_i}(t-1) + c_{t,s_i} \Big\}\\
    \mathbb{E}[N_{i,T}] & \leq \frac{8\ln T}{\Delta_i^2} +  \sum_{t=\frac{8\ln T}{\Delta_i^2}}^T ~\sum_{s_1=1}^t~\sum_{s_i=\frac{8\ln T}{\Delta_i^2}}^t 2t^{-4} \\ & ~~\leq ~~~\frac{8\ln T}{\Delta_i^2} \sum_{t=\frac{8\ln T}{\Delta_i^2}}^\infty~\sum_{s_1=1}^t~\sum_{s_i=\frac{8\ln T}{\Delta_i^2}}^t 2t^{-4}\\
    & \leq \frac{8\ln T}{\Delta_i^2} + 1 + \frac{\pi^2}{3}  ~~~~~~~~~~~~~~\tag{as $ \sum_{t=\frac{8\ln T}{\Delta_i^2}}^\infty~\sum_{s_1=1}^t~\sum_{s_i=\frac{8\ln t}{\Delta_i^2}}^t 2t^{-4} ~~~\leq ~~~ 1 + \frac{\pi^2}{3}  $ }
\end{align*}
Recall from Section \ref{sec:model}, Equation \ref{def:smab regret}, that
\begin{align*}
    \mathbb{E}[\mathcal{R}_{\mathcal{\textnormal{\ucb}}}(T)] &= \sum_{i \in [k]}{\Delta_i \cdot \mathbb{E}[N_{i,T}]}\\
    & \leq \sum_{i \neq 1} \frac{8 \ln T}{\Delta_i} + \Big(1 + \frac{\pi^2}{3} \Big) \cdot \sum_{i \in [k]}\Delta_i
\end{align*}
\end{proof}

\end{subsection}

\subsection{Distribution-free Regret Bound for \ucb1 }
\label{sec:instance independent regret of ucb}
\begin{restatable}{theorem}{ucbInstanceIndependent}
\label{thm: ucb instance independent}
For the \SMAB problem, the \ucb1 has expected (distribution-free) regret $\mathbb{E}[\mathcal{R}_{\textnormal{\ucb}}(T)] = O\big(\sqrt{T \ln T}\big)$.
\end{restatable}
\begin{proof}
Recall from Section \ref{subsec:ucbProof} that the expected cumulative regret of the \ucb1 algorithm in any round $T$ is given by $$\mathbb{E}\big[\mathcal{R}_{\textnormal{\ucb}}(T)\big] = \sum_{i \ in [k]} \Delta_i \cdot \mathbb{E}[N_{i,T}].$$
To bound the above quantity, we begin by defining the event
$$C := \Bigg\{\big|\hat{\mu}_{i}(t) - \mu_{i}\big| \leq \smash{\sqrt{\frac{2\ln T}{N_{i,t}}}}, \forall i \in [k], \forall t \leq T\Bigg\}.$$
By applying Hoeffding's inequality, and taking union bound, we get
$$\mathbb{P}\big( \Bar{C}\big) \leq \frac{2kT}{T^4} \leq \frac{2}{T^2}.$$
Next, we will bound the value of $\mathbb{E}\big[\mathcal{R}_{\textnormal{\ucb}}(T)\big]$ by conditioning on $C$ and $\Bar{C}$. Let us first bound $\mathbb{E}\big[\mathcal{R}_{\textnormal{\ucb}}(T) \big| C\big]$. Assume the event $C$ holds and some arm $i_t \neq 1$ is pulled in round $t \in [T]$. Then, by definition of \ucb1 algorithm, we have $\Bar{\mu}_{1}(t) < \Bar{\mu}_{i}(t)$. Then,
\begin{align*}
    \mu_1 - \mu_{i_t} &\leq \mu_1 - \mu_{i_t} + \Bar{\mu}_{i}(t) -\Bar{\mu}_{1}(t)\\
    &= \big(\mu_1 -\Bar{\mu}_{1}(t)\big) + \big(\Bar{\mu}_{i}(t) - \mu_{i_t} \big)
\end{align*}
Since event $C$ holds, we have
$$\mu_1 - \Bar{\mu}_{1}(t) = \mu_1 - \hat{\mu}_{1}(t-1) - \sqrt{\frac{2\ln T}{N_{i,t-1}}} \leq 0.$$
and
$$\Bar{\mu}_{i}(t) - \mu_{i_t} = \hat{\mu}_{i_t}(t-1) - \mu_{i_t} + \sqrt{\frac{2\ln T}{N_{i_t,t-1}}} \leq 2\cdot \sqrt{\frac{2\ln T}{N_{i_t,t-1}}}.$$
Therefore, 
\begin{equation}
    \label{eq:instance independent delta}
    \mu_1 - \mu_{i_t} \leq 2\cdot \sqrt{\frac{2\ln T}{N_{i_t,t-1}}}
\end{equation}
Now, consider any arm $i \in [k]$ and consider the last round $t_i \leq t$ when this arm was last pulled. Since the arm has not been pulled between $t_i$ and $t$, we know $N_{i,t_i} = N_{i,t-1}$. Hence, applying the inequality in Equation \ref{eq:instance independent delta} to arm $i$ in round $t_i$, we get
$$\mu_1 - \mu_i \leq 2\cdot \sqrt{\frac{2\ln T}{N_{i,t-1}}}, \text{~for all} t \leq T$$.
Thus, the regret in $t$ rounds is bounded by 
$$\mathcal{R}(t) = \sum_{i \in [k]} \Delta_i \cdot N_{i,t} \leq 2\sqrt{2\ln T} \cdot \sum_{i \in [k]} \smash{\sqrt{N_{i,t}}}.$$
Square root is a concave function, and hence from Jensen's inequality, we obtain
$$\sum_{i \in [k]} \sqrt{N_{i,t}} \leq \sqrt{kt}.$$
Therefor, we have
$$\mathbb{E}\big[\mathcal{R}_{\textnormal{\ucb}}(T) \big| C\big] \leq 2\sqrt{2kt\ln T}.$$
Hence, the expected cumulative regret in $t$ rounds can be bounded as
\begin{align*}
    \mathbb{E}\big[\mathcal{R}_{\textnormal{\ucb}}(T) &= \mathbb{E}\big[\mathcal{R}_{\textnormal{\ucb}}(T) \big| C\big] \mathbb{P}(C) + \mathbb{E}\big[\mathcal{R}_{\textnormal{\ucb}}(T) \big| \Bar{C}\big] \mathbb{\Bar{C}}\\
    &\leq 2\sqrt{2kt\ln T} + t \cdot \frac{2}{T^2}\\
    &= O\big( \sqrt{kt\ln T} \big), ~~~\forall t \leq T
\end{align*}
Thus, the distribution-free regret bound of \ucb1 algorithm at some time $T$ is $O(\sqrt{T \ln T})$.
\end{proof}

\section{Omitted Proofs}
\label{sec:omitted}

\subsection{Distribution-free \rRegret bound for \fucb}
\InstanceIndependentRegret*
\begin{proof}
Recall from Definition \ref{def:rRegret} our expression for the \rRegret of a \FSMAB algorithm $\mathcal{A}$. We know,
\begin{align*}
    \mathbb{E}[\mathcal{R}_{\mathcal{A}}^{r}(T)] &= \sum_{i\in [k]}{\Delta_i \cdot \Big(\mathbb{E}[N_{i,T}] - \textnormal{max}\big(0,\floor{r_i \cdot T} - \alpha \big) \Big)}\\
    & \leq k + \sum_{i\in [k]}{\Delta_i \cdot \Big(\mathbb{E}[N_{i,T}] - \textnormal{max}\big(0,r_i \cdot T - \alpha \big) \Big)}
\end{align*}
Note that, given any instance with $k$ arms, $\mu = (\mu_1, \mu_2,$ $ \ldots, \mu_k)$, and a constant $\alpha \geq 0$,
\begin{align*}
    \mathbb{E}[\mathcal{R}_{\mathcal{A}}^{r}(T)] \leq& \max_{r_i \in [0,1]^k, \sum_{i\in[k]}r_i < 1} ~~~k~~~ +\\
    %\mathbb{E}[\mathcal{R}_{\mathcal{A}}^{r}(T)] \leq& \max_{\substack{r_i \in [0,1]^k\\\sum_{i\in[k]}r_i < 1}} k +\\
    &~~~\sum_{i\in [k]}{\Delta_i \cdot \Big(\mathbb{E}[N_{i,T}] - \textnormal{max}\big(0,r_i \cdot T - \alpha \big) \Big)}\\
    \leq& ~~k + \sum_{i\in [k]}{\Delta_i \cdot \mathbb{E}[N_{i,T}]}
\end{align*}
The last inequality follows from the fact that $r_i \geq 0$ for all $i \in  k$, and $\alpha$ is a constant. This implies that the regret for any instance with given value of $r = (r_1,r_2,\ldots,r_k)$ is bounded by the regret of the same instance for $r_1 = r_2 = \ldots = r_k = 0$. But when, $r_1 = r_2 = \ldots = r_k = 0$, \FUCB is the same as \ucb1. Hence, from the distribution-free regret bound of \ucb1 (See Appendix \ref{appendix}), the result follows. Thus we can upper bound the distribution-free regret of \FUCB as $O(\sqrt{T\ln T})$.
\end{proof}

\subsection{Cost of Fairness in the \SMAB Problem}
\label{app:cost of fairness proof}

\thmAlphaRegret*
\begin{proof}
From Section \ref{sec:model}, Definition \ref{def:smab regret} we know that $\mathcal{R}_{\mathcal{A}}(T) =$ $\sum_{i\in [k]}{\Delta_i \cdot \mathbb{E}[N_{i,T}]}$ and hence, we can bound the expected regret of an algorithm by bounding the expected number of pulls of a sub-optimal arm. In particular, we want to bound the quantity $\mathbb{E}[N_{i,T}]$ for every sub-optimal arm $i \neq 1$. We do this by considering two cases dependent on how many times the arm $i$ has been pulled to satisfy the fairness constraint, i.e. on how large is the quantity $r_i \cdot T - \alpha$. 

\noindent \underline{Case 1: } Let $i \neq 1$ and $r_i \cdot T - \alpha \geq \frac{8 \ln T}{\Delta_i^2}$. Then\\
\begin{align*}
    \mathbb{E}[N_{i,T}] &\leq (r_i \cdot T - \alpha) + \sum_{t = 1}^T \mathbbm{1}\{i_t = i, N_{i,t-1} \geq r_i \cdot T - \alpha\}\\
    &\leq (r_i \cdot T - \alpha)  \\ +   \sum_{t=1}^\infty & \sum_{s_1=1}^t  \sum_{s_i=r_i \cdot T - \alpha}^t \mathbbm{1}{\Big\{ \hat{\mu}_{1,s_1}(t) + c_{t,s_1} \leq \hat{\mu}_{1,s_i}(t) + c_{t,s_i} \Big\}} \tag{Follows from Section \ref{subsec:ucbProof}}\\
\end{align*}
Since $(r_i \cdot T - \alpha) \geq \frac{8 \ln T}{\Delta_i^2}$, it follows from the proof of Theorem \ref{theorem: UCB} that $\mathbb{E}[N_{i,T}] \leq (r_i \cdot T - \alpha) + \Big( 1 + \frac{\pi^2}{3} \Big)$.

\noindent\underline{Case 2: }Let $i \neq 1$ and $r_i \cdot T - \alpha < \frac{8 \ln T}{\Delta_i^2}$\\
Then the proof of Theorem \ref{theorem: UCB} can be appropriately adapted to show that $\mathbb{E}[N_{i,T}] \leq \frac{8 \ln T}{\Delta_i^2} + \Big( 1 + \frac{\pi^2}{3} \Big)$. Hence
\begin{align*}
    r_i \cdot T - \alpha \leq \mathbb{E}[N_{i,T}] \leq \frac{8 \ln T}{\Delta_i^2} + \Big( 1 + \frac{\pi^2}{3} \Big)
\end{align*}
Then from the two cases discussed above, we can conclude that
\begin{align*}
    \mathcal{R}(T) &\leq \sum_{i \in S(T)} (r_i \cdot T - \alpha) \cdot \Delta_i + \sum_{\substack{i \notin S(T) \\ i \neq 1}}  \bigg( \frac{8 \ln T}{\Delta_i}\bigg) \\ &  + \sum_{i \in [k]}\big(1 + \frac{\pi^2}{3} \big) \cdot \Delta_i
\end{align*}
where $S(T) = \Big\{i \in [k] ~\big|~ r_i \cdot T - \alpha \geq \frac{8 \ln T}{\Delta_i^2}\Big\}$.
\end{proof}

\input{T_aware_Algo.tex}

\section{Rationale for Simulation Parameters}
\label{app: simulation param}
For evaluating the performance of our algorithm, we perform experiments on synthetic data sets as this allows for finer control on the tuning of the parameters of the experiment. In particular, we consider the following two \FSMAB instances:\\

\noindent\textbf{Instance 1: Fairness vs. Regret}
\begin{itemize}
    \item Number of arms: $k = 10$
    \item Time horizon: $T = 10^6$
    \item Mean rewards: $\mu_1 = 0.8$ and $\mu_i = \mu_1 - \Delta_i$ where $\Delta_i = 0.01i$
    \item Fairness constraint: $r = (0.05)_{i \in [k]}$
\end{itemize}
\textbf{Instance 2: FUCB vs. LFG}
\begin{itemize}
    \item Number of arms: $k = 3$
    \item Time horizon: $T = 200$
    \item Mean rewards: $\mu = (0.7,0.5,0.4)$
    \item Fairness constraints: $r = (0.2,0.3,0.25)$
\end{itemize}

Note that a large value for the time horizon $T$ significantly increases the simulation time. On the other hand, a small value for $T$ does not capture the convergence of a \SMAB algorithm. In Instance 1, we choose a sufficiently large value of $T$ so that the convergence of \FUCB is captured. On the other hand, we use smaller value of $T$ in Instance 2 because it allows us capture the maximum fairness violation at each round more clearly.

Next, we use Instance 1 to evaluate the performance of the algorithms in terms of \rRegret, whereas Instance 2 is used to evaluate the fairness guarantee. First, we consider the number of arms chosen. If the number of arms is very small, a learning algorithm will correctly identify the optimal arm very soon. On the other hand, a larger number of arms increases the simulation time required to depict the convergence behaviour of the algorithm. Our choice of $k$ in Instance 1 is sufficient to depict the behaviour of our algorithm in terms of regret without significantly increasing the simulation time. 
%%SN- It might help to drive home the point that 10 is a large enough value (I'm not 100% convinced from this alone). You can think of adding a line that says something to the effect of if k were set to 20 say the time horizon would have to be increased to 10^8 (or appropriate value) however the behaviour observed would simply scale. -- Done (i suppose)
On the other hand, we choose a small value of $k$ in Instance 2 to showcase the fairness guarantee. Keeping $k$ small allows us more flexibility in terms of choosing $r_i$'s (since $\sum_{i \in [k]}r_i < 1$). Thus, by choosing $r_i$'s of sub-optimal arms such that the number of times \FUCB is required to play these arms is significantly more than that by classic \ucb1, allows us to test the fairness guarantee of \fucb.

Our choice of the expected rewards of the arms, $\mu = (\mu)_{i \in [k]}$ in Instance 1 is such that the differences in the expected rewards of two adjacent arms is small. Consequently, the algorithm needs more time to correctly decide the optimal arm. Furthermore, we also carried out simulations with $\mu_i$'s with higher difference between them. However, our current choice of $\mu_i$'s captures the contrast in the regret performance of the \FSMAB algorithms much better. In contrast, the choice of $\mu$ in Instance 2 is because the fairness guarantee of \FUCB can be tested more rigorously when the differences in $\mu_i$'s is significant as this causes the algorithm to correctly identify the optimal arm quickly. In the standard \UCB algorithm, this would lead to the sub-optimal arms being pulled significantly fewer number of times. As a result, choosing greater values of $r_i$'s for these arms allows for more strict evaluation of the fairness guarantees of \fucb.

The fairness constraint vector $r = (r_i)_{i \in [k]}$ in Instance 1 is again selected such that it provides a clear depiction of the cost of fairness in terms of the conventional notion of regret. The choice of $r$ in Instance 2 allows for more meticulous assessment of the fairness guarantees of the two algorithms. We have also carried out the experiments with different values of the fairness constraint vector but our choice turns out to be the one suitable for the purpose of representation.

%% file: T_aware_Algo.tex
\section{Horizon-aware Algorithms}
\label{sec:tAware}
\SetAlFnt{\small}
%Notice the dependence of \rRegret on $r$. It is immediate that if  $\sum_{i \in [k]} r_i =1$ we have $\mathcal{R}_{\mathcal{A}}^{r}(T) = 0 $. Throughout this paper we assume,  without loss of generality, that $\sum_{i \in [k]} r_i  < 1$.

An algorithm that has access to time horizon $T$ and has to satisfy fairness constraints only at the end of $T$ rounds (and not uniformly at the end of all rounds) can trade-off fairness and regret more effectively. To see this, notice that in order to identify the best arm quickly it is important that an algorithm should explore the arms in the initial rounds. This observation along with Observation \ref{obs:observationOne} gives us that if the arms are pulled initially to satisfy the fairness constraints, the algorithm incurs no regret and at the same time learns the rewards from each arm. In other words the algorithm incurs no regret for first $T':= \sum_{i \in [k]} r_i \cdot T$ number of rounds. If  $r$ is such that the $T'$ is sufficient to explore each arm and find the best arm with high probability then one can pull the best arm for rest of the $T -T'$ rounds. Notice that now the fairness constraints are only satisfied after $T'$ rounds. Guided by this intuition we propose a \ucb1 based \emph{T-aware} algorithm called \TFUCB algorithm that satisfies the fairness requirement at the end of $T$ rounds and achieves logarithmic \rregret. 

\begin{algorithm}[h]
 \SetAlgoLined
 \KwIn{$[k], (r_i)_{i \in [k]}$}
 \textbf{Initialize:}\\
  $n_i \gets \max{ \big(1, r_i \cdot T}\big)$ for each $i \in [k]$\\
  $T' = \sum_{i \in [k]} n_i$ \\
  \For{$t = 1,2, \ldots, T'$} { 
    - Pull each arm $i \in [k]$ exactly $n_i$ times \\
  }
 \For{$t = T' + 1, \ldots, T$} { 
    - $i_t = \argmax_{i \in [k]} \Bar{\mu}_i(t)$ \\
    - Update $\Bar{\mu}_i(t+1)$ 
  }
  \caption{\tfucb}
  \label{alg:tfucb}
\end{algorithm}
%\end{minipage}
%\captionof{figure}{T-aware Algorithms}
%\label{alg:tawareAlg}
%\end{figure}

\noindent \textbf{\ucb1 based Algorithm (\tfucb):}
This \TFUCB algorithm knows the time horizon $T$, and effectively separates the \emph{fairness constraint satisfaction} phase and the \emph{regret minimization} phase and achieves logarithmic \rRegret in terms of $T$ with dependence on the values of the fairness fractions.
%We propose a \ucb1 based T-aware fair algorithm, \tfucb. This algorithm knows the time horizon $T$, and effectively separates the \emph{fairness constraint satisfaction} phase and the \emph{regret minimization} phase and achieves logarithmic \rRegret in terms of $T$ with dependence on the values of the fairness fractions.
\TFUCB is presented in Algorithm \ref{alg:tfucb}. Note that \TFUCB satisfies the fairness requirements of all arms at $T'$ itself, but does not provide uniform time fairness guarantee as \FUCB. Next we show that \TFUCB achieves logarithmic \rregret.
\begin{restatable}{theorem}{TFucbRegret}
\label{thm:TfucbRegret}
For \FSMAB problem, \tfucb ~has \rRegret $\mathcal{R}_{\tfucb}^{r}(T) = O(\ln T)$. In particular, its $r$-dependent regret is given by 
\begin{align*}
    \mathcal{R}_{\tfucb}^{r}(T) \leq& \Big(1 + \frac{\pi^2}{3} \Big) \cdot \sum_{i \in [k]}\Delta_i \\
    &+ \sum_{\substack{i \in S(T) \\ i \neq 1}}\Delta_i \cdot \bigg( \frac{8 \ln T}{\Delta_i^2} - r_i \cdot T \bigg)
\end{align*}
where $S(T) = \Big\{i \in [k] ~ \big|~ r_i \cdot T < \frac{8 \ln T}{\Delta_i^2}\Big\}$.
\end{restatable}
\begin{proof}
Recall $\Bar{\mu}_{i}(t) = \hat{\mu}_{i,N_{i,t-1}}(t-1) + c_{t,N_{i,t-1}}$ is the UCB estimate of the mean of arm $i$, where $\hat{\mu}_{i,N_{i,t-1}}(t-1)$ is the empirical estimate of the mean of arm $i$ when it is played $N_{i,t-1}$ in $t-1$ rounds and $c_{t,N_{i,t-1}} = \sqrt{\frac{2\ln t}{N_{i,t-1}}}$ is the confidence interval of the arm $i$ at round $t$. Similar to the proof of Theorem \ref{theorem: UCB} (\ucb1 algorithm), we upper bound the expected number of times a sub-optimal arm is pulled. We do this for each sub-optimal arm by considering two cases dependent on the number of times the sub-optimal arm is pulled in the fairness constraint satisfaction phase, i.e. in the first $T'$ rounds. %Suppose for some $t \in [T'+1,T]$, a sub-optimal arm $i \neq 1$ is pulled. Then we bound $\mathbb{E}[N_{i,T}]$ by considering the following two cases.
\newline
\underline{Case 1: } Let $i \neq 1$ and $r_i \cdot T \geq \frac{8 \ln T}{\Delta_i^2}$. Then\\
\begin{align*}
    \mathbb{E}[N_{i,T}] &\leq r_i \cdot T + \sum_{t = T'+1}^T \mathbbm{1}\{i_t = i, N_{i,t-1} \geq r_i \cdot T\}\\
    &\leq r_i \cdot T \\&+ \sum_{t=T'}^\infty\sum_{s=1}^t\sum_{s_i=r_i \cdot T}^t \mathbbm{1}{\Big\{ \hat{\mu}_{1,s}(t) + c_{t,s} \leq \hat{\mu}_{1,s_i}(t) + c_{t,s_i} \Big\}} \tag{Follows from Appendix \ref{sec:prelims}}\\
\end{align*}
Since $r_i \cdot T \geq \frac{8 \ln T}{\Delta_i^2}$, it follows from the proof of Theorem \ref{theorem: UCB} that $\mathbb{E}[N_{i,T}] \leq r_i \cdot T + \Big( 1 + \frac{\pi^2}{3} \Big)$. Hence, the expected number of pulls of a sub-optimal arm $i \neq 1$ in the regret minimization phase is $\mathbb{E}[N_{i,T}] - r_i \cdot T \leq \Big( 1 + \frac{\pi^2}{3} \Big)$. 

\noindent\underline{Case 2: }Let $i \neq 1$ and $r_i \cdot T < \frac{8 \ln T}{\Delta_i^2}$\\
Then the proof of Theorem \ref{theorem: UCB} can be appropriately adapted to show that $\mathbb{E}[N_{i,T}] \leq \frac{8 \ln T}{\Delta_i^2} + \Big( 1 + \frac{\pi^2}{3} \Big)$. Thus the expected number of pulls of a sub-optimal arm $i \neq 1$ in the regret minimization phase is 
\begin{align*}
    \mathbb{E}[N_{i,T}] - r_i \cdot T &\leq \frac{8 \ln T}{\Delta_i^2} + \Big( 1 + \frac{\pi^2}{3} \Big) - r_i \cdot T \\
    &\leq \hspace{0.5cm}\frac{8 \ln T}{\Delta_i^2} + \Big( 1 + \frac{\pi^2}{3} \Big)
\end{align*}
Suppose $S(T) = \Big\{i \in [k] ~\big|~ r_i \cdot T < \frac{8 \ln T}{\Delta_i^2}\Big\}$. Then from the two cases discussed above, we can conclude that
\begin{align*}
    \mathcal{R}_{\tfucb}^{r}(T) \leq& \big(1 + \frac{\pi^2}{3} \big) \cdot \sum_{i \in [k]}\Delta_i \\
    &+ \sum_{\substack{i \in S(T) \\ i \neq 1}}\Delta_i \cdot \bigg( \frac{8 \ln T}{\Delta_i^2}  - r_i \cdot T \bigg) 
\end{align*}

\noindent Hence, $\mathcal{R}_{\tfucb}^{r}(T) = O(\ln T)$.
\end{proof}
%Thus we showed that \TFUCB achieves $O(\ln T)$ \rRegret while (trivially) satisfying fairness when the time horizon $T$ is known to the algorithm beforehand.

%\begin{proof}[Proof Outline]
%\TFUCB does not incur any regret in the first $T'$ rounds. After $T'$, \tfucb decides which arm to play at time $t$ based on the \ucb estimates of the arms. For the UCB algorithm, we know that $\mathbb{E}[N_{i,T}] = O\big(\frac{8\ln T}{\Delta_i^2}\big)$ for any sub-optimal arm $i \neq 1$. Hence, if for any arm we have $r_i \cdot T \geq \frac{8\ln T}{\Delta_i^2}$, then that arm will be played for only a small constant number of times after $T'$ and hence the regret due to such an arm is bounded by a small value. On the other hand, if for some sub-optimal arm $i$, $r_i \cdot T < \frac{8\ln T}{\Delta_i^2}$, then we incur a regret equal to $\Delta_i$ for $\mathbb{E}[N_{i,T}] - r_i \cdot T$ rounds i.e. for at most $\frac{8\ln T}{\Delta_i^2} - r_i \cdot T$ rounds. Hence, the expected regret of \tfucb,   $\mathcal{R}_{\tfucb}^{r}(T) = O(\ln T)$. Proof is provided in Appendix \ref{sec:omitted}.
%\end{proof}